\documentclass[lettersize,journal]{IEEEtran}

\usepackage[T1]{fontenc}
\usepackage[utf8]{inputenc}
\usepackage{amsmath,amssymb,amsfonts,amsthm,mathtools}
\usepackage{dsfont}
\usepackage{algorithm}
\usepackage{algpseudocode}
\usepackage{array}
\usepackage{booktabs}
\usepackage{multirow}
\usepackage{threeparttable}
\usepackage[caption=false,font=footnotesize]{subfig}
\usepackage{textcomp}
\usepackage{stfloats}
\usepackage{url}
\usepackage{verbatim}
\usepackage{graphicx}
\usepackage{xcolor}
\usepackage{cite}
\usepackage[colorlinks=true,
linkcolor=blue,
citecolor=blue,
urlcolor=blue]{hyperref}

\def\BibTeX{{\rm B\kern-.05em{\sc i\kern-.025em b}\kern-.08em
    T\kern-.1667em\lower.7ex\hbox{E}\kern-.125emX}}

\newtheorem{lemma}{Lemma}[section]

\title{Semi-Supervised Hyperbolic Hierarchical Clustering with Set-Level Structural Priors}

\author{Junjing Zheng, Xinyu Zhang, Xiangfeng Qiu, Chengliang Song, and Weidong Jiang%
\thanks{The authors are with the College of Electronic Science and Technology, National University of Defense Technology, Changsha 410073, China.}%
\thanks{Corresponding author: Xinyu Zhang (e-mail: zhangxinyu90111@163.com).}%
}

\begin{document}

\maketitle

\begin{abstract}
Semi-supervised hierarchical clustering aims to learn a tree structure consistent with data patterns and user-provided supervision. Supervision is usually given as leaf-level relations, such as pairwise must-link/cannot-link constraints or triplet-wise must-link-before constraints. Although useful for regulating local sample relations, such supervision does not directly indicate which samples should form coherent subtrees. Consequently, the non-leaf structure of the learned tree may deviate from the hierarchical organization preferred by ground-truth labels. To address this limitation, we propose a semi-supervised hyperbolic hierarchical clustering method with set-level structural priors. The main contribution is to introduce sets as basic modeling units for hierarchy learning. Each set denotes samples expected to cohere within a subtree and is induced from leaf-level supervision together with a learned constraint-consistent similarity structure. These sets act as soft structural priors for subtree-level supervision, allowing supervision to guide non-leaf hierarchy formation beyond local leaf-level relations. Specifically, we first learn constraint-consistent embeddings to obtain a reliable set partition, then construct constraint-induced sets and estimate inter-set similarities to form set-level structural priors. Finally, these priors are incorporated into a hyperbolic hierarchy objective for continuous tree optimization. Experiments on eleven benchmark datasets and ablation studies show that the proposed method consistently improves label consistency over representative hierarchical clustering baselines while also enhancing similarity-based tree quality.

\end{abstract}


\section{Introduction}
Hierarchical clustering (HC) is a fundamental unsupervised learning paradigm for organizing data into multi-level nested structures, which provides richer structural interpretability than flat clustering~\cite{ran2023comprehensive, ren_deep_2024}. However, purely unsupervised HC may produce hierarchies that are inconsistent with domain knowledge or user expectations. Semi-supervised hierarchical clustering (SSHC)~\cite{gonzalez-almagro_semi-supervised_2025, cai_review_2023} mitigates this limitation by integrating prior knowledge into the clustering process. Specifically, such priors are expressed as sample relation constraints that guide the construction of more meaningful dendrograms.

Existing SSHC methods mainly employ pairwise constraints~\cite{miyamoto2010semi, hamasuna2010semi, miyamoto2011constrained} or triplet-wise constraints~\cite{zhao2010hierarchical, kestler2006effects}. Pairwise must-link (ML) and cannot-link (CL) constraints~\cite{SSSE} specify whether two samples should belong to the same or different clusters, and have been widely used in semi-supervised clustering~\cite{semi_multicons,wagstaff2001constrained}. Nevertheless, pairwise constraints do not explicitly encode hierarchical preferences~\cite{bade_creating_2008}, since each constraint only describes a relation between two samples rather than the granularity at which they should be grouped in a hierarchy. To express relative merge orders, triplet-wise must-link-before (MLB) constraints are introduced~\cite{Bade,Bade2013}, which specify that one sample pair should be merged before another pair. Although MLB constraints are more directly related to hierarchical clustering, they are usually harder to obtain than pairwise constraints in practical scenarios. Therefore, when the available supervision is given as sparse pairwise constraints, only a limited number of reliable MLB relations can be constructed. More importantly, both pairwise and triplet-wise constraints still operate on individual leaf nodes and thus provide leaf-level supervision, as illustrated in Fig.~\ref{fig:leaf_level}. They can influence subtree formation only indirectly, and therefore provide limited structural guidance for learning the hierarchy.

Motivated by this observation, we aim to introduce supervision at a higher structural level than individual leaves, referred to as subtree-level supervision. Directly imposing hard subtree constraints is impractical, because the subtree structure is unknown before hierarchical optimization and should remain adjustable during training. To this end, we use the set as a natural carrier of such supervision. In this work, a set is a group of samples induced from available prior knowledge and a learned similarity structure. It is not assumed to be a fixed cluster or a completed subtree during clustering. Instead, it serves as a soft prior indicating that its members are expected to cohere within a subtree in the learned hierarchy. In this sense, the proposed set-level structural priors provide a soft realization of subtree-level supervision. These priors offer dual regularization by encouraging internal coherence within each set and informing boundary relations between different sets during dynamic tree optimization. To operationalize the above idea, we embed the hierarchy formation process in hyperbolic space. Hyperbolic geometry is naturally suited to modeling tree-like structures~\cite{HypHC} and provides a continuous Riemannian manifold where structural priors can be incorporated into gradient-based optimization. Prior work has demonstrated that hyperbolic hierarchical clustering can be formulated as a differentiable objective~\cite{HypHC, HypCSE}. Building on these insights, we propose \emph{Semi-supervised Hyperbolic Hierarchical Clustering with Set-level Structural Priors}, which treats constraint-induced sets as the basic modeling units and optimizes their hierarchical relations on the Poincar\'e model. The core idea is shown in Fig.~\ref{fig:subtree_level}.

Concretely, the proposed method consists of three coordinated stages. First, a novel set-aware representation learning scheme calibrates the embedding space through a combination of unsupervised and semi-supervised losses, jointly considering raw data geometry and known pairwise constraints. This stage provides a constraint-consistent similarity structure for subsequent stages. Second, the hierarchical clustering set construction stage constructs constraint-induced sets using ML-induced closures, mutual $k$-nearest neighbors, and shared-neighbor consistency. It further computes weak-connectivity inter-set similarities, so that the set partition and the inter-set similarities together form set-level structural priors. Third, the set-based hyperbolic hierarchy optimization stage uses these priors in a set-triplet objective by introducing intra-set hyperbolic lowest common ancestor (LCA) representations. The sample embeddings are further optimized on the hyperbolic manifold, and the final tree is decoded from the optimized embeddings. Through these stages, sparse pairwise supervision is organized into set-level structural guidance that supports subtree formation during continuous hierarchy optimization. Our contributions are summarized as follows:

\begin{figure*}[t]
	\centering
	\subfloat[Leaf-level supervision]{%
		\includegraphics[width=0.465\textwidth]{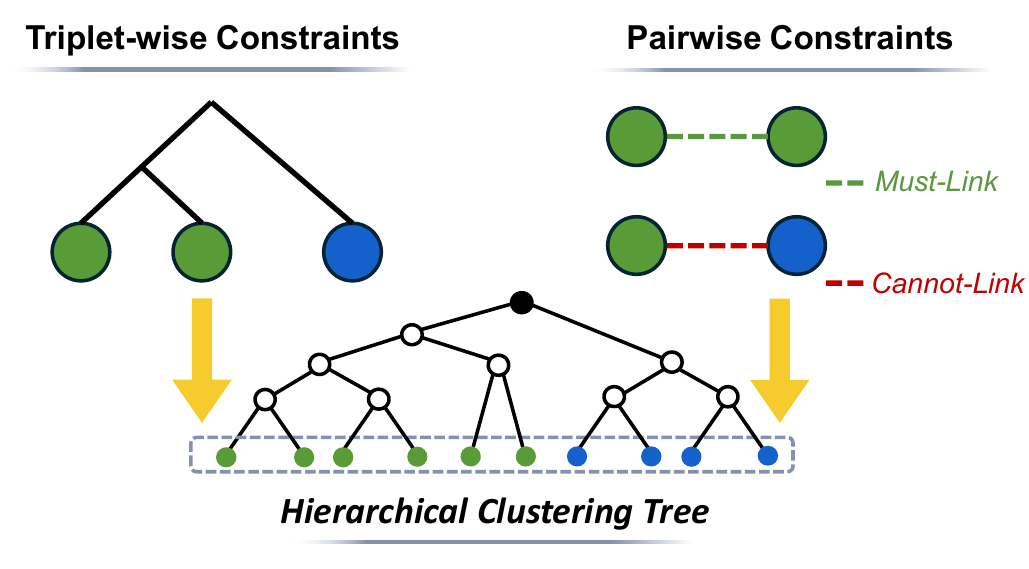}%
		\label{fig:leaf_level}
	}
	\hfil
	\subfloat[Our subtree-level supervision]{%
		\includegraphics[width=0.48\textwidth]{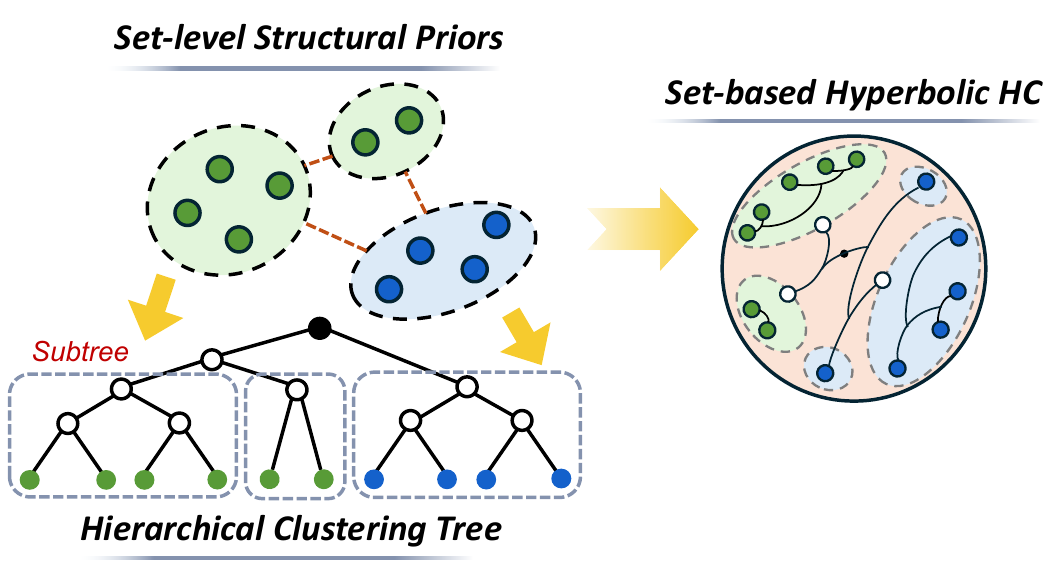}%
		\label{fig:subtree_level}
	}
	\caption{Differences between leaf-level supervision and the proposed subtree-level supervision.}
	\label{fig:pairwise_constraints_mismatch}
\end{figure*}

\begin{itemize}
	\item A constraint-consistent similarity structure is learned for set partitioning.
	The proposed representation learning scheme balances intrinsic data geometry with available pairwise supervision, providing a more reliable similarity structure for set partitioning.
	
	\item Leaf-level supervision is lifted to subtree-level supervision.
	By constructing constraint-induced sets and inter-set similarities, sparse leaf-level relations are organized into soft structural priors for subtree formation beyond local leaf-level relations.
	
	\item Set-level structural priors are incorporated into continuous hyperbolic hierarchy learning.
	The proposed set-based objective extends point-triplet optimization to set-triplet relations, allowing structural priors to guide subtree formation and produce more label-consistent hierarchies.
	
\end{itemize}

\section{Preliminaries}
\subsection{Notations}

Given a dataset $\mathcal{X}=\{\mathbf{x}_1,\mathbf{x}_2,\ldots,\mathbf{x}_n\}$, we formulate hierarchical clustering as learning a rooted binary tree $\mathcal{T}$ with $n$ leaves, where each leaf corresponds to one data point and each internal node corresponds to a cluster. For two leaves $i$ and $j$, their lowest common ancestor (LCA) is denoted by $i\vee j$. The subtree rooted at $i\vee j$ is denoted by $\mathcal{T}[i\vee j]$, and $\text{leaves}(\mathcal{T}[i\vee j])$ denotes its leaf set. For hyperbolic embeddings, $\mathbf{z}_i\vee \mathbf{z}_j$ denotes the hyperbolic LCA of two embedded samples.

We consider semi-supervised hierarchical clustering with pairwise constraints. Let $l_i$ be the cluster label of $\mathbf{x}_i$. The ML and CL constraint sets are denoted by $\mathcal{M}$ and $\mathcal{C}$, respectively, where each $(\mathbf{x}_i,\mathbf{x}_j)\in\mathcal{M}$ satisfies $l_i=l_j$, and each $(\mathbf{x}_i,\mathbf{x}_j)\in\mathcal{C}$ satisfies $l_i\ne l_j$. After feature representation, we construct a set collection $\mathcal{S}=\{S_1,S_2,\ldots,S_P\}$ from the constraints and the learned similarity structure of the latent feature space. The corresponding local structure is encoded by an undirected graph $G=\{V,E,\mathbf{W}\}$, where $V=\{v_1,v_2,\ldots,v_n\}$ contains vertices corresponding to feature points, $E$ is the edge set, and $\mathbf{W}$ stores pairwise similarities. $\mathbf{z}^{\mathrm{intra}}_S$ denotes the intra-set hyperbolic LCA representation of a set $S\in\mathcal{S}$, which is referred to as the intra-set LCA hereafter.

\subsection{Hyperbolic Geometry and Hyperbolic HC}

\subsubsection{Poincar\'e model}
Hyperbolic geometry has constant negative curvature and is well suited for representing tree-like hierarchical structures~\cite{HypHC}. In this work, we use the Poincar\'e model with curvature $-1$, defined as
\begin{equation}
	\mathbb{B}^d=\{\mathbf{x}\in\mathbb{R}^d:\Vert\mathbf{x}\Vert<1\}
\end{equation}
For $\mathbf{x},\mathbf{y}\in\mathbb{B}^d$, the M\"obius addition is defined as~\cite{HypHC}
\begin{equation}
	\mathbf{x}\oplus\mathbf{y}
	=
	\frac{
		(1+2\langle \mathbf{x},\mathbf{y}\rangle+\Vert\mathbf{y}\Vert^2)\mathbf{x}
		+
		(1-\Vert\mathbf{x}\Vert^2)\mathbf{y}
	}{
		1+2\langle \mathbf{x},\mathbf{y}\rangle+\Vert\mathbf{x}\Vert^2\Vert\mathbf{y}\Vert^2
	}.
\end{equation}
The hyperbolic distance between $\mathbf{x}$ and $\mathbf{y}$ is then given by
\begin{equation}
	\label{HyperDist}
	d^{h}(\mathbf{x}, \mathbf{y}) =
	2 \operatorname{artanh}\!\left(\Vert -\mathbf{x}\oplus\mathbf{y}\Vert\right).
\end{equation}
Denoting $\mathbf{o}$ as the origin, it reduces to $d^{h}_{\mathbf{o}}(\mathbf{x})=2\operatorname{artanh}(\Vert \mathbf{x}\Vert)$ when $\mathbf{y}=\mathbf{o}$. To map vectors between Euclidean tangent spaces and the Poincar\'e model, we use the exponential and logarithmic maps:
\begin{equation}
	\label{project}
	\operatorname{exp}_{\mathbf{x}}(\mathbf{u})=\mathbf{x}\oplus\left(\operatorname{tanh}\left(\frac{\lambda_{\mathbf{x}}\Vert \mathbf{u}\Vert}{2}\right)\frac{\mathbf{u}}{\Vert \mathbf{u}\Vert}\right),
\end{equation}
where $\mathbf{u}\in\mathbb{R}^d$ and $\lambda_{\mathbf{x}}=\frac{2}{1-\Vert\mathbf{x}\Vert^2}$, and
\begin{equation}
	\label{projectback}
	\operatorname{log}_{\mathbf{x}}(\mathbf{y})=\frac{2}{\lambda_{\mathbf{x}}}\operatorname{arctanh}(\Vert -\mathbf{x}\oplus\mathbf{y}\Vert)\frac{-\mathbf{x}\oplus\mathbf{y}}{\Vert -\mathbf{x}\oplus\mathbf{y}\Vert},
\end{equation}
where $\mathbf{y}\in\mathbb{B}^d$. In particular, when the base point is the origin, the two maps reduce to
$\operatorname{exp}_{\mathbf{o}}(\mathbf{u})=\operatorname{tanh}\left(\Vert \mathbf{u}\Vert\right)\frac{\mathbf{u}}{\Vert \mathbf{u}\Vert}$
and
$\operatorname{log}_{\mathbf{o}}(\mathbf{y})=\operatorname{arctanh}(\Vert \mathbf{y}\Vert)\frac{\mathbf{y}}{\Vert \mathbf{y}\Vert}$.

\subsubsection{Klein model}
The Klein model $\mathbb{K}^d$ is another representation of hyperbolic space with the same open ball domain as the Poincar\'e model. Its primary advantage in this context is that geodesics are Euclidean straight lines, thereby facilitating the use of convex combinations to compute the intra-set LCA. A point $\mathbf{x}\in\mathbb{B}^d$ can be mapped to the Klein model by
\begin{equation}
	\label{eq:poincare-to-klein}
	\mathbf{k} = \frac{2\mathbf{x}}{1 + \|\mathbf{x}\|^2},
\end{equation}
and a Klein point $\mathbf{k}\in\mathbb{K}^d$ can be mapped back to the Poincar\'e model by
\begin{equation}
	\label{eq:klein-to-poincare}
	\mathbf{x} = \frac{\mathbf{k}}{1 + \sqrt{1 - \|\mathbf{k}\|^2}}.
\end{equation}
These transformations will be used to derive the continuous intra-set LCA in Section~\ref{sec:intra-set_LCA}.

\subsubsection{Hyperbolic Hierarchical Clustering}
Dasgupta's cost provides a global objective for similarity-based hierarchical clustering~\cite{dasgupta}:
$\mathcal{L}_{\text{Dasgupta}}(\mathcal{T};w)=\sum_{ij}w_{ij}\vert \text{leaves}(\mathcal{T}[i \vee j])\vert$.
It favors trees in which similar data points have LCAs farther from the root than dissimilar points. Wang et al.~\cite{triplesDA} reformulate this objective into a triplet-based form, but the resulting optimization remains combinatorial. HypHC~\cite{HypHC} addresses this issue by embedding leaf nodes into hyperbolic space and optimizing a continuous relaxation of the hierarchical clustering objective.
Given hyperbolic embeddings $\mathcal{Z}=\{\mathbf{z}_1,\mathbf{z}_2,\ldots,\mathbf{z}_n\}$, HypHC is defined as
\begin{equation}
	\label{HypHC}
	\begin{aligned}
		\mathcal{L}_{\text{HypHC}}(\mathcal{Z};w,t)
		&= \sum_{i,j,k}\Big(
		w_{ij}+w_{ik}+w_{jk} \\
		&\quad
		- w_{ijk}(\mathcal{Z};w,t)
		\Big)
		+2\sum_{i,j}w_{ij},                                      \\[1mm]
		w_{ijk}(\mathcal{Z};w,t)
		&=(w_{ij},w_{ik},w_{jk})\cdot
		\sigma_{t}\!\left(
		\begin{array}{c}
			d^{h}_{\mathbf{o}}(\mathbf{z}_i\vee \mathbf{z}_j)\\
			d^{h}_{\mathbf{o}}(\mathbf{z}_i\vee \mathbf{z}_k)\\
			d^{h}_{\mathbf{o}}(\mathbf{z}_j\vee \mathbf{z}_k)
		\end{array}
		\right).
	\end{aligned}
\end{equation}
Here, $\mathbf{z}_i \vee \mathbf{z}_j$ denotes the hyperbolic LCA, defined as the point on the geodesic between $\mathbf{z}_i$ and $\mathbf{z}_j$ that is closest to the origin~\cite{HypHC}. The function $\sigma_{t}(\cdot)$ is the scaled softmax with temperature $t>0$.

\subsection{Related Works}

Most SSHC methods based on pairwise constraints follow an agglomerative or graph-based paradigm, and mainly differ in how constraints are translated into merge decisions. COBRA~\cite{van2018cobra} over-clusters data into super-instances and uses pairwise constraints to rule out illegal merges. CECH~\cite{SCECH} directly modifies the similarity matrix by assigning maximum similarity to ML pairs and minimum similarity to CL pairs. Semi-Multicons~\cite{semi_multicons} encodes constraints into a consensus-based objective and seeks a tree satisfying closed-pattern constraints. SSSE~\cite{SSSE} represents pairwise constraints as a relation graph and integrates them into structural entropy minimization, where hierarchical clustering is performed through discrete stretching and compressing operators. These methods can effectively exploit pairwise supervision, but they still mainly introduce limited local corrections to pairwise similarities or merge orders. As a result, supervision is propagated to subtree formation only indirectly. In contrast, our method first learns a more reliable similarity structure through set-aware representation learning, which balances the available supervision with the intrinsic neighborhood structure of the data. It further uses sets as the basic units for hierarchy learning, allowing leaf-level supervision to act as subtree-level guidance and produce more label-consistent hierarchies.

Recently, hyperbolic geometry has been widely used for modeling hierarchical data because of its negative curvature and exponential volume growth~\cite{Nickel2017}. A series of hyperbolic neural models have been developed to support representation learning on hyperbolic manifolds~\cite{HGCN, Liu2019HGNN, Mathieu2019, HMHGNN2025,HBNN2025}, and have been applied to clustering~\cite{HypCSE, MHCN, WAH}, image retrieval~\cite{MixGeo}, and session-based recommendation~\cite{liuenhancing2024}. For hierarchical clustering, HypHC~\cite{HypHC} formulates tree learning as continuous optimization of leaf embeddings in hyperbolic space, while HypCSE~\cite{HypCSE} relaxes structural entropy by using soft assignments over hyperbolic LCAs. These methods demonstrate the effectiveness of continuous hyperbolic optimization, but they are mainly unsupervised and optimize leaf-level hierarchy objectives. As a result, the learned hierarchies may not necessarily align with domain knowledge or user expectations. Our method complements this line of work by introducing sets as the carriers of supervision into hyperbolic HC, so that set-level structural priors can guide hierarchy learning at the upper levels of the tree and produce more label-consistent hierarchies.

\section{Semi-supervised HypHC}

\subsection{Framework Overview}

Given a dataset $\mathcal{X}=\{\mathbf{x}_1,\mathbf{x}_2,\ldots,\mathbf{x}_n\}$ with known pairwise ML and CL constraints $\mathcal{M}$ and $\mathcal{C}$, our goal is to transform sparse leaf-level supervision into soft subtree-level guidance and use it to learn a hyperbolic hierarchy. Pairwise constraints are used as the practical supervision source because they are often easier to obtain than triplet-wise constraints. Technically, sparse leaf-level pairwise constraints are transformed into soft subtree-level supervision carried by constraint-induced sets. These sets do not define fixed subtrees in advance; instead, they provide set-level structural priors that guide subtree formation and merging during continuous hierarchy optimization on the Poincar\'e model.

As shown in Fig.~\ref{overview}, the proposed method follows a preparation--construction--optimization pipeline. First, the set-aware representation learning stage combines unsupervised and semi-supervised losses to learn constraint-consistent hyperbolic embeddings, yielding a constraint-consistent similarity structure for reliable set partition. Second, the hierarchical clustering set construction stage partitions samples into constraint-induced sets and computes inter-set similarities. This is the stage where set-level structural priors are explicitly formed from constraints and the learned similarity structure. Third, the set-based HypHC stage employs the set partition and inter-set similarities in a set-triplet objective to further optimize leaf embeddings on the Poincar\'e model, where intra-set LCAs are computed from geodesic convex hulls as continuous representations of subtree roots. In this way, set-level structural priors are translated into subtree merge preferences, so that the decoded tree better aligns with available supervision.

\begin{figure*}
    \centering
    \includegraphics[
    width=1\linewidth,
    trim=0 5pt 0 12pt,
    clip
    ]{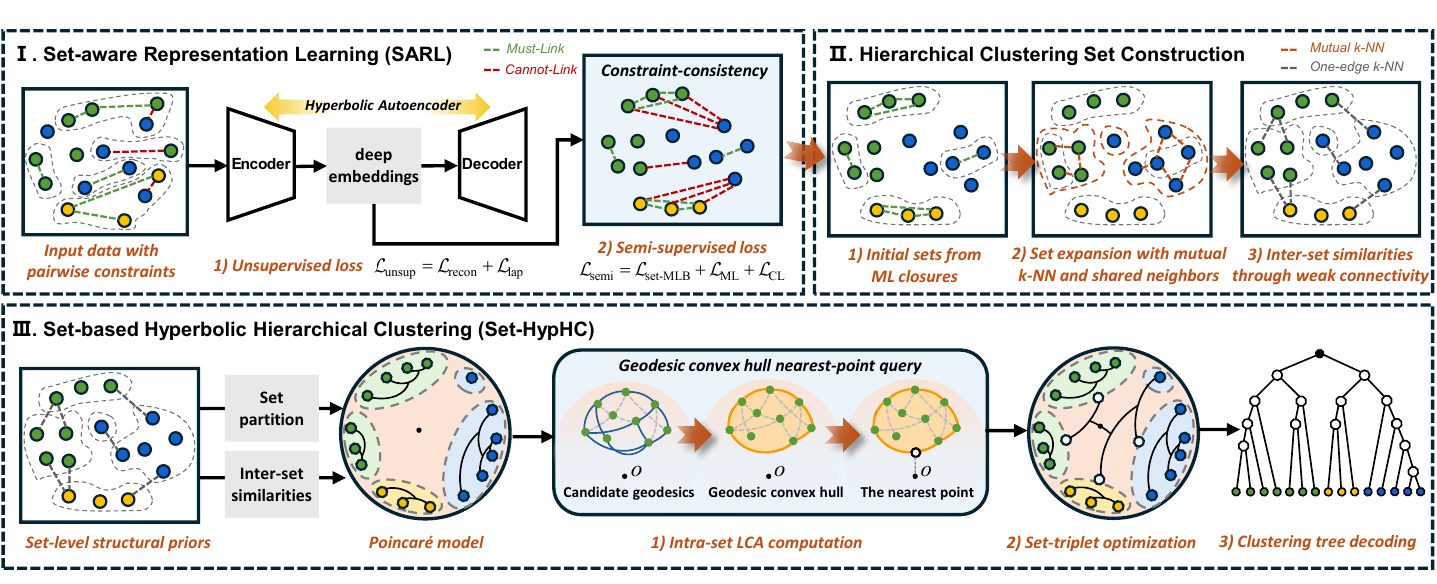}
    \caption{Overall framework of the proposed method.}
    \label{overview}
\end{figure*}

\subsection{Set-Aware Representation Learning}

\subsubsection{Constraint-consistent Hyperbolic Embeddings}

This stage reduces the risk that raw-feature similarities scatter constraint-related samples or introduce noisy local proximity, providing constraint-consistent embeddings. We adopt a hyperbolic 
autoencoder, implemented with a  multi-layer perceptron (MLP)-based Euclidean autoencoder and a Poincar\'e mapping layer. Given an input sample $\mathbf{x}_i$, we use a Euclidean encoder 
$f_{\mathrm{enc}}$ parameterized by $\theta_{\mathrm{enc}}$ to map it to
\begin{equation}
	\mathbf{z}_i^e=f_{\mathrm{enc}}(\mathbf{x}_i;\theta_{\mathrm{enc}}),
\end{equation}
which is then projected onto the Poincar\'e model:
\begin{equation}
	\mathbf{z}_i^h=
	\exp_{\mathbf{o}}(\mathbf{z}_i^e),
\end{equation}
where the constraint-related objectives are imposed. 
The embedding is further mapped back to the tangent space and fed 
into a decoder $f_{\mathrm{dec}}$, which is parameterized by 
$\theta_{\mathrm{dec}}$ and symmetric to the encoder:
\begin{equation}
	\tilde{\mathbf{x}}_i
	=
	f_{\mathrm{dec}}
	\left(
	\log_{\mathbf{o}}(\mathbf{z}_i^h);
	\theta_{\mathrm{dec}}
	\right).
\end{equation}
This design preserves the reconstruction capability of autoencoders while allowing supervision to shape the hyperbolic embedding space. The objective of this stage is
\begin{equation}
	\label{overall loss}
	\mathcal{L}_{\mathrm{rep}}
	=
	\mathcal{L}_{\mathrm{semi}}
	+
	\mathcal{L}_{\mathrm{unsup}},
\end{equation}
where $\mathcal{L}_{\mathrm{semi}}$ encodes supervision, and $\mathcal{L}_{\mathrm{unsup}}$ preserves input semantics and local smoothness.

\subsubsection{Semi-Supervised Terms}
From the available pairwise constraints, we first extract implicit triplet-wise merge preferences from ML/CL closures and formulate a novel set-aware MLB term. Explicit ML contraction and CL separation terms are further introduced as complementary pairwise regularizers, and together they define the semi-supervised loss.

The standard MLB criterion specifies that a sample pair $(\mathbf{x}_i,\mathbf{x}_j)$ should merge earlier than pairs involving another sample $\mathbf{x}_k$, i.e., $d_{ij}<d_{ik}$ and $d_{ij}<d_{jk}$. Taking $\mathbf{x}_i$ as the anchor, this can be written as a ranking objective:
\begin{equation}
	\max \ d_{ik}-d_{ij}
	\Leftrightarrow
	\min \log \left(1+\exp(d_{ij}-d_{ik})\right).
\end{equation}
The last term can be rewritten as
\begin{equation}
	\log \left(1+\exp(d_{ij}-d_{ik})\right)
	=
	-\log
	\frac{\exp(-d_{ij})}
	{\exp(-d_{ij})+\exp(-d_{ik})},
\end{equation}
which corresponds to a logistic ranking form. When $\mathbf{x}_j$ is also used as an anchor, the two MLB inequalities $d_{ij}<d_{ik}$ and $d_{ij}<d_{jk}$ can be jointly encouraged.

Based on this ranking view, we derive a set-aware MLB term from pairwise ML/CL closures. For an anchor $\mathbf{x}_i$, let $\mathcal{M}_i$ and $\mathcal{C}_i$ denote its ML and CL closures derived from $\mathcal{M}$ and $\mathcal{C}$, respectively. These closures provide preliminary supervision for representation learning and are utilized as initial sets in the set construction stage. The set-aware MLB loss is defined as
\begin{equation}
	\label{set-MLB}
	\mathcal{L}_{\text{set-MLB}}
	=
	-\sum_{i\in{\mathcal{A}}}
	\log
	\frac{
		\sum_{j\in\mathcal{M}_i}\exp(-d_{ij})
	}{
		\sum_{j\in\mathcal{M}_i}\exp(-d_{ij})
		+
		\sum_{k\in\mathcal{C}_i}\exp(-d_{ik})
	},
\end{equation}
where $\mathcal{A}$ is the set of anchors with both ML and CL closures. This loss compares the anchor with an ML closure and a CL closure, rather than with two isolated samples. Therefore, it extracts relative merge preferences from pairwise constraints in a more structured and set-aware manner, without requiring explicit triplet-wise annotations.

We further introduce explicit ML contraction and CL separation terms. For the ML part, let $\mathcal{M}^{\text{hard}}$ be the subset of the farthest ML pairs, with $|\mathcal{M}^{\text{hard}}|=r_1|\mathcal{M}|$ and $0<r_1<1$. The ML loss is
\begin{equation}
	\label{ML loss}
	\mathcal{L}_{\mathrm{ML}}
	=
	\frac{1}{|\mathcal{M}^{\text{hard}}|}
	\sum_{(i,j)\in\mathcal{M}^{\text{hard}}}
	d_{ij}^{2}.
\end{equation}
For the CL part, let $\mathcal{A}_{\mathrm{CL}}$ denote anchors that participate only in CL relations, and let $\mathcal{C}^{\text{hard}}$ be the subset of nearest CL samples, with $|\mathcal{C}^{\text{hard}}|=r_2|\mathcal{C}|$ and $0<r_2<1$. The CL loss is
\begin{equation}
	\label{CL loss}
	\mathcal{L}_{\mathrm{CL}}
	=
	\frac{1}{|\mathcal{A}_{\mathrm{CL}}|}
	\sum_{i\in\mathcal{A}_{\mathrm{CL}}}
	\frac{1}{|\mathcal{C}^{\text{hard}}|}
	\sum_{k\in\mathcal{C}^{\text{hard}}}
	\frac{1}{1+d_{ik}^{2}},
\end{equation}
which penalizes CL samples that are excessively close to the anchor. The semi-supervised loss can be given as
\begin{equation}
	\label{semi loss}
	\mathcal{L}_{\mathrm{semi}}
	=
	\mathcal{L}_{\text{set-MLB}}
	+
	w_{\mathrm{ML}}\mathcal{L}_{\mathrm{ML}}
	+
	w_{\mathrm{CL}}\mathcal{L}_{\mathrm{CL}},
\end{equation}
where $w_{\mathrm{ML}}$ and $w_{\mathrm{CL}}$ are balancing coefficients.

Following mixed-geometry representation learning~\cite{MixGeo}, the distance $d_{ij}$ is defined by combining Euclidean and hyperbolic distances:
\begin{equation}
	\label{mixed distance}
	d_{ij}
	=
	\frac{d^{e}(\mathbf{z}_i^{e},\mathbf{z}_j^{e})}{\tau_e}
	+
	\lambda
	\frac{d^{h}(\mathbf{z}_i^{h},\mathbf{z}_j^{h})}{\tau_h},
\end{equation}
where $\tau_e$ and $\tau_h$ are temperature parameters, and $\lambda$ controls the contribution of the hyperbolic term. This mixed-geometric property helps preserve local feature similarity while incorporating hyperbolic structural information, leading to a similarity structure that is more suitable for subsequent set construction.

\subsubsection{Unsupervised Terms}

In addition to $\mathcal{L}_{\mathrm{semi}}$, we use two unsupervised regularizations to preserve the intrinsic data geometry. The reconstruction loss keeps the learned embeddings from deviating excessively from the input semantics, while the Laplacian regularization encourages local smoothness in the embedding space and stabilizes the neighborhood structure used for set construction. Thus, we have
\begin{equation}
	\mathcal{L}_{\mathrm{unsup}}
	=
	\mathcal{L}_{\mathrm{rec}}
	+
	\mathcal{L}_{\mathrm{lap}}.
\end{equation}

Let $\tilde{\mathbf{x}}_i$ denote the reconstruction of $\mathbf{x}_i$. The reconstruction loss is defined as
\begin{equation}
	\label{reconstruction}
	\mathcal{L}_{\mathrm{rec}}
	=
	\frac{1}{n}
	\sum_{i=1}^{n}
	\left\|
	\mathbf{x}_i-\tilde{\mathbf{x}}_i
	\right\|_2^2 .
\end{equation}

For Laplacian regularization, we construct a similarity matrix on Euclidean embeddings using the radial basis function (RBF) kernel:
\begin{equation}
w_{ij}
=
\exp\left(
-\frac{(d^e_{ij})^2}{2\sigma_e^2}
\right),
\end{equation}
where $d^e_{ij}=d^e(\mathbf{z}_i^e,\mathbf{z}_j^e)$ and $\sigma_e$ is the median of $d^e_{ij}$. Let $\mathbf{D}$ be the degree matrix with $D_{ii}=\sum_j w_{ij}$, and let $\mathbf{L}=\mathbf{D}-\mathbf{W}$ be the graph Laplacian. The Laplacian loss is
\begin{equation}
	\label{laplacian}
	\mathcal{L}_{\mathrm{lap}}
	=
	\frac{
		\operatorname{Tr}\!\left(
		(\mathbf{Z}^{e})^{\top}\mathbf{L}\mathbf{Z}^{e}
		\right)
	}{
		\sum_{i,j}w_{ij}
	},
\end{equation}
where $\mathbf{Z}^{e}=[\mathbf{z}_1^e,\mathbf{z}_2^e,\ldots,\mathbf{z}_n^e]^{\top}$.

Combining the supervised and unsupervised terms, the final objective of SARL is
\begin{equation}
	\label{rep-final}
	\mathcal{L}_{\mathrm{rep}}
	=
	\mathcal{L}_{\text{set-MLB}}
	+
	w_{\mathrm{ML}}\mathcal{L}_{\mathrm{ML}}
	+
	w_{\mathrm{CL}}\mathcal{L}_{\mathrm{CL}}
	+
	\mathcal{L}_{\mathrm{rec}}
	+
	\mathcal{L}_{\mathrm{lap}}.
\end{equation}

\subsection{Hierarchical Clustering Set Construction}

\subsubsection{Constraint-induced Set Construction}

This stage constructs an index-based set partition 
$\mathcal{S}=\{S_p\}_{p=1}^{P}$ from the available supervision and 
the learned constraint-consistent similarity structure, where each 
$S_p=\{i_{p,1},i_{p,2},\ldots,i_{p,n_p}\}\subseteq\{1,\ldots,n\}$ 
contains sample indices. The corresponding embeddings are accessed as 
$\{\mathbf{z}_i:i\in S_p\}$. These sets are not fixed subtrees, but soft structural 
units that encourage coherent subtree formation in the subsequent set-based 
HypHC stage.

Before set construction, we define the inter-sample similarity using the RBF kernel and hyperbolic distance:
\begin{equation}
	\label{hyperbolic_similarity}
	w_{ij}^{h}
	=
	\exp\left(
	-\frac{(d_{ij}^{h})^2}{2\sigma_h^2}
	\right),
\end{equation}
where $d_{ij}^{h}=d^h(\mathbf{z}_i,\mathbf{z}_j)$, and $\sigma_h$ is the median of $d_{ij}^{h}$. Based on the learned embeddings, we construct a constraint-aware $k$-nearest neighbor ($k$-NN) graph $G=\{V,E,\mathbf{W}\}$. For nodes involved in ML constraints, we first search for their $k$ nearest neighbors after excluding ML neighbors, and then reintroduce ML edges into the graph. Meanwhile, edges between CL samples are removed if they appear in the neighborhood graph due to high similarity. This construction extends local neighborhood relations from ML closures to unconstrained samples while preventing CL-conflicting edges from providing erroneous guidance.

We then construct the set partition $\mathcal{S}$ in three steps. First, ML closures are used as initial sets. For a node $i$, let $\mathcal{N}_i^{\mathrm{cand}}$ denote its candidate neighbors in the constraint-aware $k$-NN graph. For a leftover sample not included in any ML set, if at least $p_1$ nodes in $\mathcal{N}_i^{\mathrm{cand}}$ belong to the same ML set, it is absorbed into that set. This expansion extends ML-induced local structures without merging different ML components.

Second, the remaining unassigned samples are grouped according to mutual $k$-NN and shared-neighbor consistency. Two unassigned nodes $i$ and $j$ are grouped into the same new set if they are mutual neighbors, i.e., $i\in\mathcal{N}_j^{\mathrm{cand}}$ and $j\in\mathcal{N}_i^{\mathrm{cand}}$, and they share at least $s$ common neighbors:
$|\mathcal{N}_i^{\mathrm{cand}}\cap \mathcal{N}_j^{\mathrm{cand}}|\ge s.$
This rule ensures that newly formed sets contain samples that are not only close to each other, but also supported by similar local neighborhoods.

Finally, isolated samples are assigned through a voting mechanism. For an isolated node $i$, let $\mathcal{N}_i^{\mathrm{top}}$ denote its top $p_2$ nearest neighbors in $\mathcal{N}_i^{\mathrm{cand}}$ that already belong to existing sets. The node is assigned to the set with the highest vote count in $\mathcal{N}_i^{\mathrm{top}}$. Samples that remain unassigned after voting are treated as singleton sets. The resulting partition $\mathcal{S}$ is used to form set-level structural priors for subsequent set-based hierarchy optimization. The complete workflow is summarized in Algorithm~\ref{alg:set-construction}.

\begin{algorithm}[t]
	\caption{Hierarchical Clustering Set Construction}
	\label{alg:set-construction}
	\begin{algorithmic}[1]
		\Require Learned embeddings $\{\mathbf{z}_i\}_{i=1}^{n}$, ML constraints $\mathcal{M}$, CL constraints $\mathcal{C}$, neighbor size $k$, thresholds $p_1,s$, voting size $p_2$
		\Ensure Set partition $\mathcal{S}=\{S_1,\ldots,S_P\}$
		
		\State Compute hyperbolic similarities $w_{ij}^{h}$ by Eq.~\eqref{hyperbolic_similarity}.
		\State Build a constraint-aware $k$-NN graph $G=(V,E,\mathbf{W})$ by excluding ML neighbors during neighbor search, reinserting ML edges, and removing CL-conflicting edges.
		\State Initialize $\mathcal{S}$ as the connected components induced by $\mathcal{M}$.
\State Let $U\leftarrow V\setminus\bigcup_{S_p\in\mathcal{S}}S_p$ be the unassigned node set.
		
		\State Expand existing ML sets: assign each $i\in U$ to an existing set if at least $p_1$ nodes in $\mathcal{N}_i^{\mathrm{cand}}$ vote for the same set; update $U$.
		
		\State Build $G_U=(U,E_U)$ as an auxiliary subgraph of $G$, where $(i,j)$ is added if $i$ and $j$ are mutual $k$-NNs and
		$|\mathcal{N}_i^{\mathrm{cand}}\cap\mathcal{N}_j^{\mathrm{cand}}|\ge s$.
		\State Add all non-singleton connected components of $G_U$ to $\mathcal{S}$, and update $U$.
		
		\State For each remaining $i\in U$, assign it to the set receiving the largest vote from its top-$p_2$ assigned neighbors.
		\State Add still-unassigned nodes as singleton sets.
		
		\State \Return $\mathcal{S}$.
	\end{algorithmic}
\end{algorithm}

\subsubsection{Inter-Set Similarities}

After obtaining the set partition $\mathcal{S}$, we define inter-set similarities to describe boundary relations among sets. The key idea is weak connectivity over candidate boundary links. Instead of using the strongest inter-set links, which may be dominated by a few noisy or accidental neighbors, we estimate the inter-set similarity from the relatively weak portion of these candidate links. Thus, a high similarity is assigned only when the boundary connectivity between two sets remains consistently strong.

For two sets $S_A$ and $S_B$, let
$
E_{AB}=\{(i,j)\in E \mid i\in S_A,\ j\in S_B\}
$
be the inter-set edges in the constraint-aware $k$-NN graph, and let
$
P_{AB}=\{(i,j)\mid i\in S_A,\ j\in S_B\}
$
be all cross-set sample pairs. For a candidate pair set $Q$, let $\operatorname{Low}_{r_3}(Q)$ denote the bottom $r_3$ proportion of pairs in $Q$ ranked by $w_{ij}^{h}$. The similarity between $S_A$ and $S_B$ is defined as
\begin{equation}
	\label{eq:set-sim}
	w_{AB}
	=
	\begin{cases}
		\displaystyle
		\frac{1}{|\operatorname{Low}_{r_3}(E_{AB})|}
		\sum_{(i,j)\in \operatorname{Low}_{r_3}(E_{AB})}
		w_{ij}^{h},
		& E_{AB}\neq \emptyset, \\[3mm]
		\displaystyle
		\frac{1}{|\operatorname{Low}_{r_3}(P_{AB})|}
		\sum_{(i,j)\in \operatorname{Low}_{r_3}(P_{AB})}
		w_{ij}^{h},
		& E_{AB}= \emptyset .
	\end{cases}
\end{equation}
When inter-set graph edges exist, the similarity is estimated from the weakest inter-set edges; otherwise, it is computed from the weakest cross-set pairs. This definition makes the inter-set similarity sensitive to boundary connectivity, thereby providing more reliable guidance for subtree merge preferences. Together with the set partition $\mathcal{S}$, the inter-set similarities $\{w_{AB}\}$ form the set-level structural priors used in the subsequent set-triplet hierarchy objective.

\subsection{Set-Based Hyperbolic Hierarchical Clustering}
\label{sec:set_hyphc}
\subsubsection{Revisiting HypHC from a Point-Triplet Perspective}

As defined in Eq.~\eqref{HypHC}, HypHC~\cite{HypHC} relaxes HC into a continuous optimization problem by using hyperbolic LCAs to approximate the LCAs of a rooted binary tree. For a leaf pair $(i,j)$, the hyperbolic LCA $\mathbf{z}_i\vee\mathbf{z}_j$ serves as a continuous proxy for their common ancestor, and its depth $d^h_{\mathbf{o}}(\mathbf{z}_i\vee\mathbf{z}_j)$ reflects the relative merge level of this pair. At the point-triplet level, HypHC compares the three LCA depths induced by $(i,j,k)$ and encourages the pair with the highest similarity to obtain the deepest LCA, i.e., to merge earliest. In this way, leaf embeddings are optimized to match the merge order of a binary tree without direct combinatorial tree search.

To incorporate supervision into this point-triplet framework, a direct strategy is to adjust the pairwise similarities according to ML/CL constraints. However, this still injects supervision through local point-pair relations. Such local correction may weaken the global coherence of hyperbolic optimization, since the hierarchy is determined by the joint geometry of all leaf embeddings rather than isolated pairwise relations. A more suitable strategy is to make sparse supervision act on the structural units that determine the hierarchy, i.e., intermediate non-leaf structures corresponding to common ancestors of leaf subsets.

This motivates us to use constraint-induced sets as the modeling units for supervised hierarchy optimization. Since a non-leaf node defines a subtree and thus corresponds to a subset of leaves, each set provides a soft leaf subset that is expected to form a coherent subtree during optimization, without being fixed as a subtree in advance. We therefore extend HypHC from point-level triplets to set-level triplets. Instead of comparing the LCA depths of three sample pairs, the set-based objective compares the LCA depths associated with three sets. This allows set-level structural priors to shape intermediate hierarchical structures while preserving the continuous optimization advantage of HypHC.

\subsubsection{Continuous intra-set LCAs in Hyperbolic Space}
\label{sec:intra-set_LCA}
To extend HypHC from point triplets to set triplets, we need a continuous representation for each set in hyperbolic space. In the original HypHC, $\mathbf{z}_i\vee\mathbf{z}_j$ is defined as the point on their geodesic that is closest to the origin. For a set $S$ with multiple samples, we define its intra-set LCA $\mathbf{z}^{\mathrm{intra}}_{S}$ as the closest point to the origin within the geodesic convex hull of its member embeddings:
\begin{equation}
	\label{eq:intra_set_lca}
	\mathbf{z}^{\mathrm{intra}}_{S}
	=
	\arg\min_{\mathbf{z}\in \operatorname{conv}_{\mathbb{B}}(\{\mathbf{z}_i\}_{i\in S})}
	d^h_{\mathbf{o}}(\mathbf{z}),
\end{equation}
where $\operatorname{conv}_{\mathbb{B}}(\cdot)$ denotes the geodesic convex hull in the Poincar\'e model, and $d^h_{\mathbf{o}}(\mathbf{z})=2\operatorname{artanh}(\|\mathbf{z}\|)$.

The intra-set LCA serves as a continuous proxy for the internal node that represents the subtree encouraged by $S$. The following lemma shows that this representation approximates the corresponding root of a genuine subtree with a bounded error, which allows us to directly optimize the subtrees that contain multiple leaf nodes to guide
the tree structure, rather than operating solely on leaf-level pairwise relations.

\begin{lemma}[Intra-set LCA approximates tree LCA]
	Let $S\subseteq\{1,\ldots,n\}$ be a finite set of sample indices and
	$\mathbf{z}^{\mathrm{intra}}_{S}$ be its intra-set LCA as defined in Eq.~\eqref{eq:intra_set_lca}.
Let $\mathcal{T}$ be a rooted binary tree whose leaf set contains the leaves corresponding to $S$ such that, for all $i,j$,
$d_{\mathcal{T}}(i,j)\le d(\mathbf{z}_i,\mathbf{z}_j)\le d_{\mathcal{T}}(i,j)+C_n$ ($C_n=\delta\cdot O(n)$ with $\delta>0$), and let
	$\tau_S=\operatorname{LCA}_{\mathcal{T}}(S)$ be the LCA of the leaves corresponding to $S$ in $\mathcal{T}$. Then
	\begin{equation}
		d_{\mathcal{T}}(0,\tau_S)-\frac{1}{2}C_n
		\le
		d^h_{\mathbf{o}}(\mathbf{z}^{\mathrm{intra}}_{S})
		\le
		d_{\mathcal{T}}(0,\tau_S)+C_n.
	\end{equation}

\end{lemma}

\begin{proof}
	The proof, together with the auxiliary definitions and lemmas used in the derivation, is provided in the appendix.
\end{proof}

Directly solving Eq.~\eqref{eq:intra_set_lca} in the Poincar\'e model is inconvenient as geodesics are nonlinear. We therefore use the Klein model, where geodesics are Euclidean straight lines and the geodesic convex hull becomes the standard Euclidean convex hull. Let
\begin{equation}
\label{eq:z2k}
	\mathbf{k}_i=\frac{2\mathbf{z}_i}{1+\|\mathbf{z}_i\|^2}
\end{equation}
be the Klein representation of $\mathbf{z}_i$. Then Eq.~\eqref{eq:intra_set_lca} can be equivalently solved in the Klein model, as the following lemma shows:

\begin{lemma}
	Let $\{\mathbf{k}_i\}_{i\in S}\subset\mathbb{K}^d$ be the Klein embeddings of a set of Poincar\'e points. The intra-set LCA in Eq.~\eqref{eq:intra_set_lca} is obtained by solving
	\begin{equation}
		\label{eq:klein_intra_lca}
		\mathbf{k}^{\mathrm{intra}}_{S}
		=
		\arg\min_{\mathbf{k}\in\operatorname{conv}(\{\mathbf{k}_i\}_{i\in S})}
		\|\mathbf{k}\|^2,
	\end{equation}
	and mapping the solution back to the Poincar\'e model as
\begin{equation}
\label{eq:k2z}
\mathbf{z}^{\mathrm{intra}}_{S}
=
\frac{
	\mathbf{k}^{\mathrm{intra}}_{S}
}{
	1+\sqrt{1-\|\mathbf{k}^{\mathrm{intra}}_{S}\|^2}
}.
\end{equation}

\end{lemma}

\begin{proof}
	In the Klein model, the distance from the origin is $d_{\mathbb{K}}(\mathbf{o},\mathbf{k})=\operatorname{artanh}(\|\mathbf{k}\|)$, which is strictly increasing with $\|\mathbf{k}\|$. Thus, minimizing the hyperbolic distance to the origin is equivalent to minimizing the Euclidean norm. Since geodesics in the Klein model are Euclidean straight lines, the geodesic convex hull is mapped to the Euclidean convex hull, which gives Eq.~\eqref{eq:klein_intra_lca}. The mapping back to the Poincar\'e model follows from Eq.~\eqref{eq:klein-to-poincare}.
\end{proof}

In practice, we write $\mathbf{k}=\sum_{i\in S}\alpha_i\mathbf{k}_i$, 
where $\alpha_i\ge 0$ and $\sum_{i\in S}\alpha_i=1$. The objective $\|\sum_{i\in S}\alpha_i\mathbf{k}_i\|^2$ is a convex quadratic function of $\alpha$, and we solve it using unrolled gradient descent with a softmax parameterization. The BB step size~\cite{bbstep} is used to accelerate convergence. The procedure is summarized in Algorithm~\ref{alg:intra-set-lca}.
In implementation, the optimization of $\alpha$ in Algorithm~\ref{alg:intra-set-lca} is used as a numerical solver for computing the intra-set LCA. The obtained $\mathbf{z}^{\mathrm{intra}}_{S}$ is then used by the set-level objective. Its role in gradient propagation is discussed in the next subsection.
When $S$ contains a single point, $\mathbf{z}^{\mathrm{intra}}_{S}$ reduces to that point. When $S$ contains two points, it coincides with the standard hyperbolic LCA used in HypHC. Thus, the intra-set LCA is a natural set-level extension of the pairwise LCA.

\begin{algorithm}[t]
\caption{Intra-Set LCA via Klein Model Minimization (with BB step size)}
\label{alg:intra-set-lca}
\begin{algorithmic}[1]
\Require Poincar\'e embeddings $\{\mathbf{z}_i\}_{i \in S}$, number of steps $T$
\Ensure Intra-set LCA $\mathbf{z}_S^{\text{intra}} \in \mathbb{B}^d$

\State Map each $\mathbf{z}_i$ to Klein model: $\mathbf{k}_i \leftarrow \frac{2\mathbf{z}_i}{1 + \|\mathbf{z}_i\|^2}$ \Comment{Eq.~\eqref{eq:z2k}}
\State Initialize logits $\mathbf{u} \leftarrow \mathbf{0} \in \mathbb{R}^{|S|}$
\For{$t = 1$ to $T$}
    \State $\boldsymbol{\alpha} \leftarrow \operatorname{softmax}(\mathbf{u})$
    \State $\mathbf{k} \leftarrow \sum_{i \in S} \alpha_i \mathbf{k}_i$
    \State Compute gradient $\mathbf{g} = \nabla_{\mathbf{u}}\|\mathbf{k}\|^2$
    \State Update $\mathbf{u}$ via gradient descent with BB step size \cite{bbstep}
\EndFor
\State $\mathbf{k}_S^{\text{intra}} \leftarrow \sum_{i \in S} \operatorname{softmax}(\mathbf{u})_i \mathbf{k}_i$
\State Map back to Poincar\'e: $\mathbf{z}_S^{\text{intra}} \leftarrow \frac{\mathbf{k}_S^{\text{intra}}}{1 + \sqrt{1 - \|\mathbf{k}_S^{\text{intra}}\|^2}}$ \Comment{Eq.~\eqref{eq:k2z}}
\State \Return $\mathbf{z}_S^{\text{intra}}$
\end{algorithmic}
\end{algorithm}
%

\subsubsection{Set-Level Loss and Optimization}

With the intra-set LCA defined above, we formulate the set-level extension of the HypHC objective. Let $\mathcal{Z}=\{\mathbf{z}_1,\ldots,\mathbf{z}_n\}$ be the trainable sample embeddings on the Poincar\'e model. Given the set partition $\mathcal{S}$ and inter-set similarities $\{w_{AB}\}$ obtained from set construction, the set-level HypHC loss operates on set triplets $(A,B,C)$ in direct analogy to the point-level formulation:
\begin{equation}
	\label{eq:set-HypHC}
	\begin{aligned}
		\mathcal{L}_{\mathrm{set\text{-}HypHC}}
		&= \sum_{A,B,C}
		\bigl(
		w_{AB}+w_{AC}+w_{BC}
		- w^{\mathrm{set}}_{ABC}
		\bigr)
		+ \mathcal{R}, \\[6pt]
		\mathcal{R} &= 2\sum_{A,B}w_{AB}, \\[6pt]
		w^{\mathrm{set}}_{ABC}
		&= (w_{AB},w_{AC},w_{BC})
		\cdot
		\sigma_t
		\!\begin{pmatrix}
			d^h_{\mathbf{o}}(\mathbf{z}^{\mathrm{intra}}_{A}\vee \mathbf{z}^{\mathrm{intra}}_{B})\\
			d^h_{\mathbf{o}}(\mathbf{z}^{\mathrm{intra}}_{A}\vee \mathbf{z}^{\mathrm{intra}}_{C})\\
			d^h_{\mathbf{o}}(\mathbf{z}^{\mathrm{intra}}_{B}\vee \mathbf{z}^{\mathrm{intra}}_{C})
		\end{pmatrix},
	\end{aligned}
\end{equation}
where $\sigma_t$ is the scaled softmax with temperature $t$, and $\vee$ denotes the standard hyperbolic LCA operator for point pairs.

In essence, Eq.~\eqref{eq:set-HypHC} uses the inter-set similarities as merge-preference weights and optimizes the sample embeddings $\mathcal{Z}$ through dynamically computed intra-set LCAs. At each iteration, $\mathbf{z}^{\mathrm{intra}}_S$ is recomputed from the current embeddings of the samples in $S$ via Algorithm~\ref{alg:intra-set-lca}, and serves as a set-level summary rather than an independent learnable set embedding. The pair of sets with the highest inter-set similarity is encouraged to merge earlier in the learned hierarchy, so that the induced subtree structure better matches the set-level structural priors. Importantly, point-level ordering is not discarded; it is refined implicitly through intra-set LCAs and the subsequent optimization of member embeddings.

In practice, we do not enumerate all $O(P^3)$ set triplets. Following HypHC~\cite{HypHC}, 
we first enumerate all set pairs and then randomly sample the remaining third set 
for each pair, yielding $O(P^2)$ sampled set triplets per epoch. This provides a tractable stochastic approximation to Eq.~\eqref{eq:set-HypHC}.

\paragraph{Gradient flow and dynamic subtree optimization}
A critical aspect of this formulation is how set-level supervision affects individual samples. In Algorithm~\ref{alg:intra-set-lca}, the intra-set LCA is obtained by solving the simplex weights over the Klein embeddings. Specifically, after the inner optimization, the Klein representation of the intra-set LCA is constructed as
\begin{equation}
	\mathbf{k}^{\mathrm{intra}}_{S}
	=
	\sum_{i\in S}\alpha_i^{*}\mathbf{k}_i ,
\end{equation}
where $\alpha_i^{*}$ denotes the optimized convex weight of sample $i$ in set $S$. Although the inner optimization used to obtain $\alpha^{*}$ is treated as a detached numerical solver, the final construction above remains differentiable with respect to the member embeddings under fixed $\alpha^{*}$. After mapping $\mathbf{k}^{\mathrm{intra}}_{S}$ back to the Poincar\'e model, the gradient of $\mathcal{L}_{\mathrm{set\text{-}HypHC}}$ with respect to $\mathbf{z}^{\mathrm{intra}}_{S}$ can be propagated back to every $\mathbf{z}_i$ with $i\in S$.

This design allows set-level merge preferences to align different sets and reshape the internal geometry of each set. A constraint-induced set therefore does not define a fixed subtree, but specifies the scope over which the structural prior acts. Through inter-set relations, such priors further propagate across sets and influence the global tree structure. Compared with point-level triplets, the proposed formulation provides more structured and globally coordinated supervision for subtree formation.

\paragraph{Degeneracy to Point-Level HypHC}
The proposed set-level loss naturally subsumes the original point-level HypHC as a special case. When each set contains exactly one sample, we have $\mathbf{z}^{\mathrm{intra}}_{S}=\mathbf{z}_i$ and $w_{AB}=w_{ij}$, reducing Eq.~\eqref{eq:set-HypHC} to the standard HypHC objective in Eq.~\eqref{HypHC}. Thus, the proposed formulation retains the original point-level model while enabling subtree-level supervision when larger sets are available.

\paragraph{Tree Decoding}
After optimization, the final binary tree is decoded from the optimized hyperbolic embeddings using the greedy decoding algorithm of HypHC~\cite{HypHC}.

\subsection{Complexity Analysis}

Let $n$ be the number of samples, $d$ be the embedding dimension, and $P$ be the number of constructed sets. Let $q$ denote the number of effective constraint-related pairs used in SARL after transitive closure and hard-pair sampling.

The main computational cost comes from three parts. First, SARL evaluates 
a dense Laplacian term based on a precomputed pairwise similarity matrix, 
together with the constraint-related losses. This stage costs 
$O(E_{\mathrm{rep}}(n^2d+qd))$, where the one-time pairwise similarity 
computation and the standard MLP costs are absorbed into the leading term. Second, set construction 
requires full hyperbolic similarities and a constraint-aware $k$-NN graph. 
Lower-order operations, including ML connected components, neighborhood 
expansion, and shared-neighbor checking, are omitted from the leading-order 
summary. The inter-set similarity computation uses graph edges when 
available and falls back to all cross-set pairs otherwise. Let 
$N_{\mathrm{fb}}$ denote the number of fallback pairs, with 
$N_{\mathrm{fb}}\le n^2$ in the worst case. Third, Set-HypHC recomputes 
intra-set LCAs for all sets at each epoch and evaluates the sampled 
set-triplet loss over $O(P^2)$ triplets, giving 
$O(E_{\mathrm{hc}}(Tnd+P^2d))$.

Therefore, the dominant time complexity is
\begin{equation}
	\mathcal{O}\!\left(
	E_{\mathrm{rep}}(n^2d+qd)
	+
	N_{\mathrm{fb}}
	+
	E_{\mathrm{hc}}(Tnd+P^2d)
	\right).
\end{equation}
The memory cost is dominated by storing full pairwise similarities or 
distances, requiring $O(n^2)$ memory, together with lower-order storage 
for embeddings, the $k$-NN graph, and inter-set similarities.

\section{Experiments}

\subsection{Experimental details}
\paragraph{Datasets}
We evaluate the proposed method on eleven real-world benchmark datasets with diverse sample sizes, feature dimensions, and numbers of classes. The datasets include Yale, ORL, Wine, Breast-Cancer, Isolet1, Australian, 
OpticalDigits, Spambase, COIL100, USPS, and PenDigits. 
These datasets are collected from commonly used clustering benchmark 
repositories, including the Deng Cai dataset repository~\cite{DengCaiData} 
and the UCI Machine Learning Repository~\cite{UCIRepository}. All datasets are standardized by subtracting the mean and dividing by the standard deviation.

\paragraph{Competitive methods}
We compare the proposed method with both semi-supervised and unsupervised HC baselines. The semi-supervised baselines include SSSE~\cite{SSSE}, Semi-Multicons~\cite{semi_multicons} (SM), and COBRA~\cite{van2018cobra}. The unsupervised baselines include HypHC~\cite{HypHC} and HypCSE~\cite{HypCSE}.

\paragraph{Implementation details}
For all experiments, we perform $10$ independent runs. In each run, ML and CL constraint pairs are randomly generated, with the number of each type set to $20\%$ of the total number of samples. To provide all semi-supervised methods with consistent supervision, we apply transitive closure to the generated pairwise constraints before feeding them into any supervised method. For COBRA \cite{van2018cobra}, active querying is disabled, and the same pre-generated ML and CL constraints are used as input. All experiments are conducted on a workstation equipped with a 2.20 GHz CPU, 64 GB RAM, and an NVIDIA GeForce RTX 3070 GPU.

For our method, the encoder has three layers with dimensions $d_{\mathrm{in}}$--$500$--$20$, where $d_{\mathrm{in}}$ is the input feature dimension. The embedding dimension is therefore fixed to $20$. The autoencoder is optimized using Adam with a learning rate of $0.001$ and a dropout rate of $0.2$ for $500$ epochs. In the mixed-geometric distance, the mixing coefficient $\lambda$ is set to $2$, and the temperature parameters $\tau_e$ and $\tau_h$ are set to $0.05$ and $0.2$, respectively, following~\cite{MixGeo}. The hard-pair sampling ratios are fixed as $r_1=0.1$ and $r_2=0.3$. The loss weights are set based on the effective class-wise sample scale. We use $w_{\mathrm{ML}}=0.0001$ and $w_{\mathrm{CL}}=10$ for Yale, ORL, Wine, Breast-Cancer, and Isolet1, and $w_{\mathrm{ML}}=0.001$ and $w_{\mathrm{CL}}=100$ for Australian, OpticalDigits, Spambase, COIL100, USPS, and PenDigits.
During set construction, the neighbor size $k$ for the $k$-NN graph is selected from $\{5,10,15,20\}$ by grid search. The shared-neighbor threshold is set to $s=\lfloor k/3\rfloor$, the expansion threshold for absorbing leftover samples into ML sets is set to $p_1=\lceil k/2\rceil$ and the voting size for singleton assignment is set to $p_2=10$. The sample proportion $r_3$ for weak-connectivity similarity is set to $0.5$. In the set-based HypHC stage, the Poincar\'e model radius is fixed to $1$ and the curvature is fixed to $-1$. The sample embeddings optimized in Set-HypHC are initialized by the hyperbolic embeddings learned in the first stage. The number of steps for solving the intra-set LCA is set to $10$. The model is trained using RAdam with a learning rate of $0.005$ for $50$ epochs. Let $P$ be the number of constructed sets; we uniformly sample $O(P^2)$ unordered set triplets to compute the set-level HypHC loss.

For the baseline methods, we use the hyperparameter settings provided in their released implementations. RBF-based pairwise similarities are computed from the standardized original features for all compared methods. For consistency, the Laplacian regularization term in our representation learning stage is also computed on a similarity graph constructed in the same manner. The embedding dimension of the Poincar\'e model in HypHC is set to $2$, following its original setting~\cite{HypHC}.

\paragraph{Evaluation metric}
Following prior work on HC, we use Dendrogram Purity (DP)~\cite{dp} as the primary evaluation metric. For a given hierarchical tree and ground-truth labels, DP is defined as the average purity of the subtrees rooted at the LCAs of all leaf pairs that share the same label, measuring how well ground-truth flat clusters are preserved in the learned hierarchy. We use the efficient DP implementation provided by HypCSE~\cite{HypCSE}. For methods involving iterative optimization, we report the best DP value attained during training for each run.

\begin{table}[t]
\centering
\caption{Detailed statistics of datasets.}
\label{tab:datasets}
\begin{tabular}{lccc}
\toprule
\textbf{Datasets} & \textbf{\#Data points} & \textbf{\#Features} & \textbf{\#Classes} \\
\midrule
Yale               & 165                    & 1,024               & 15 \\
ORL                & 400                    & 1,024               & 40 \\
Wine               & 175                    & 13                  & 3  \\
Breast-Cancer (Breast)     & 683                    & 10                  & 2  \\
Australian         & 690                    & 14                  & 2  \\
Isolet1            & 1,560                  & 617                 & 26 \\
OpticalDigits (Optical)     & 1,797                  & 64                  & 10 \\
Spambase (Spam)          & 4,601                  & 57                  & 2  \\
COIL100            & 7,200                  & 1,024               & 100\\
USPS               & 9,298                  & 256                 & 10 \\
PenDigits (Pen)         & 10,992                 & 16                  & 10 \\
\bottomrule
\end{tabular}
\end{table}

\subsection{Comparison}
\label{sec:comparison}
The comparison results are reported in Table~\ref{tab:comparison}. Our method achieves the highest DP on all datasets, showing consistent advantages over both semi-supervised and unsupervised baselines.

Existing baselines are strongly affected by how supervision is injected and how the hierarchy is initialized. Unsupervised hyperbolic methods such as HypHC~\cite{HypHC} and HypCSE~\cite{HypCSE} learn hierarchies without using pairwise supervision. Although HypCSE enhances graph representations through graph neural encoding and graph structure learning, both methods may still produce hierarchies that deviate from the hierarchical organization reflected by the ground-truth labels. SM~\cite{semi_multicons} and COBRA~\cite{van2018cobra} use pairwise supervision, but their hierarchies still depend on hard intermediate structures. Specifically, SM~\cite{semi_multicons} relies on closed-pattern constraints to guide consensus clustering, while COBRA~\cite{van2018cobra} first forms super-instances and then rules out illegal merges using pairwise constraints. Once inaccurate early partitions or super-instances are formed, later irreversible merging steps are difficult to correct.

SSSE~\cite{SSSE} is the strongest baseline because it explicitly incorporates pairwise constraints into graph-based structural entropy minimization. Compared with the other baselines, SSSE~\cite{SSSE} introduces supervision in a softer form by modifying graph-level relations through regularization terms. However, this guidance still mainly acts on pairwise relations, which is a local correction and may affect unconstrained regions unpredictably. In contrast, our method first learns a constraint-consistent embedding space, then constructs constraint-induced sets as carriers of set-level structural priors, and finally optimizes the hierarchy on a continuous hyperbolic manifold. This organizes soft supervision at the set level and makes it act coherently across the whole procedure, leading to better hierarchical clustering performance.

\newcommand{\std}[1]{$\pm$ #1}

\begin{table*}[t]
	\centering
	\caption{Comparison of Dendrogram Purity (DP in \%) across datasets.}
	\label{tab:comparison}
	\renewcommand{\arraystretch}{1.08}
	\setlength{\tabcolsep}{2.2pt}
	
	\resizebox{\textwidth}{!}{%
		\begin{tabular}{l|ccccccccccc}
			\toprule
			\textbf{Method} 
			& \textbf{Yale} & \textbf{Wine} & \textbf{ORL} & \textbf{Breast} 
			& \textbf{Isolet1} & \textbf{Optical} & \textbf{Australian} 
			& \textbf{Spam} & \textbf{COIL100} & \textbf{USPS} & \textbf{Pen} \\
			\midrule
			Ours 
			& \textbf{45.79}\std{3.35} 
			& \textbf{97.45}\std{1.52} 
			& \textbf{71.08}\std{2.85} 
			& \textbf{97.55}\std{0.46} 
			& \textbf{54.23}\std{2.51} 
			& \textbf{90.21}\std{1.38} 
			& \textbf{79.63}\std{1.61} 
			& \textbf{83.62}\std{2.64} 
			& \textbf{81.79}\std{0.70} 
			& \textbf{81.06}\std{4.31} 
			& \textbf{94.25}\std{2.24} \\
			
			SSSE 
			& \underline{37.82}\std{1.97} 
			& 90.86\std{1.03} 
			& \underline{57.66}\std{1.49} 
			& \underline{96.79}\std{0.11} 
			& 45.48\std{0.50} 
			& 83.88\std{0.03} 
			& 75.51\std{2.15} 
			& 57.44\std{1.36} 
			& \underline{61.23}\std{0.72} 
			& \underline{64.68}\std{0.80} 
			& 79.16\std{2.79} \\
			
			SM 
			& 25.43\std{2.48} 
			& 81.54\std{2.09} 
			& 19.43\std{1.49} 
			& 92.59\std{4.07} 
			& 22.00\std{0.69} 
			& 28.15\std{1.19} 
			& \underline{77.47}\std{1.86} 
			& \underline{73.84}\std{1.42} 
			& N/A 
			& N/A 
			& N/A \\
			
			COBRA 
			& 28.24\std{2.39} 
			& 80.28\std{7.71} 
			& 44.86\std{1.60} 
			& 92.32\std{5.11} 
			& \underline{47.54}\std{2.85} 
			& 54.95\std{4.31} 
			& 65.54\std{3.74} 
			& 59.90\std{1.00} 
			& 39.84\std{0.92} 
			& 41.64\std{2.50} 
			& 51.36\std{2.57} \\
			
			HypHC 
			& 21.32\std{2.62} 
			& 89.69\std{2.02} 
			& 14.62\std{0.62} 
			& 93.80\std{0.38} 
			& 16.42\std{1.25} 
			& 33.50\std{2.45} 
			& 68.47\std{1.75} 
			& 56.67\std{0.30} 
			& 5.31\std{0.31} 
			& 26.75\std{2.36} 
			& 39.18\std{2.19} \\
			
			HypCSE 
			& 20.87\std{1.15} 
			& \underline{92.42}\std{0.26} 
			& 31.55\std{0.24} 
			& 96.52\std{0.53} 
			& 28.47\std{2.15} 
			& \underline{86.65}\std{0.25} 
			& 74.75\std{1.55} 
			& 67.02\std{3.31} 
			& 48.50\std{1.76} 
			& 47.69\std{1.48} 
			& \underline{80.24}\std{0.67} \\
			\bottomrule
		\end{tabular}%
	}
	
	\vspace{1.2mm}
	\begin{minipage}{0.98\textwidth}
		\footnotesize
\emph{Note:} Bold and underlined values indicate the best and second-best results, respectively. N/A indicates that the method did not terminate within 3 hours.
	\end{minipage}
\end{table*}

\subsection{Ablation}

We conduct an ablation study to evaluate the contribution of SARL and Set-HypHC. To assess Set-HypHC, we compare it with the original point-level HypHC~\cite{HypHC} under the same input representations. When SARL is used, both HypHC~\cite{HypHC} and Set-HypHC are initialized with the same hyperbolic embeddings learned by SARL and use the same hyperbolic similarity computation. When SARL is removed, similarities are computed from standardized raw features using RBF kernels with Euclidean distances. All configurations use the same embedding dimension of $20$.

All experiments are repeated $10$ times with independently sampled constraints. For each run, we record the highest DP and the corresponding Dasgupta's Cost (DC)~\cite{dasgupta} from the same tree, and report the averages over all runs. DP measures label consistency, whereas DC evaluates similarity-based tree quality; lower DC represents better tree quality. The cost is computed using the hyperbolic similarity in Eq.~\eqref{hyperbolic_similarity}.

The results are shown in Table~\ref{tab:ablation}. First, even without SARL, Set-HypHC improves DP over point-level HypHC~\cite{HypHC} on most datasets, showing that set-level structural priors can directly guide subtree merging by optimizing set-level triplet relations, thereby producing tree structures that better align with label preferences. Second, SARL brings a substantial gain even before HC optimization, confirming that representation learning is necessary before constructing reliable constraint-induced sets. Third, given SARL embeddings, Set-HypHC further improves DP and consistently reduces DC compared with point-level HypHC~\cite{HypHC}. This indicates that the set-level objective improves label alignment while preserving similarity-based tree quality. 

\begin{table*}[t]
	\centering
	\caption{Ablation study results.}
	\label{tab:ablation}
	\fontsize{7.5pt}{9pt}\selectfont
	\setlength{\tabcolsep}{3.8pt}
	\renewcommand{\arraystretch}{0.92}
	\begin{tabular}{cc c|cccccc}
		\toprule
		\textbf{SARL} & \textbf{HC Method} & \textbf{Metric}
		& \textbf{Yale} & \textbf{Isolet1} & \textbf{Optical}
		& \textbf{Australian} & \textbf{Spam} & \textbf{USPS} \\
		\midrule
		& HypHC      & DP (\%) & 25.73 & 23.39 & 47.00 & 58.40 & 56.39 & 26.62 \\
		& Set-HypHC  & DP (\%) & 32.09 & 42.09 & 61.87 & 66.89 & 55.50 & 36.85 \\
		\checkmark & --         & DP (\%) & 43.31 & 46.45 & 89.24 & 69.80 & 81.96 & 78.78 \\
		\checkmark & HypHC      & DP (\%) & 44.12 & 47.11 & 89.95 & 71.17 & 82.79 & 79.43 \\
		\checkmark & Set-HypHC  & DP (\%) & \textbf{45.79} & \textbf{54.23} & \textbf{90.21} & \textbf{79.63} & \textbf{83.62} & \textbf{81.06} \\
		\midrule
		& & Cost scale & $\times10^{6}$ & $\times10^{9}$ & $\times10^{9}$
		& $\times10^{8}$ & $\times10^{10}$ & $\times10^{11}$ \\
		\checkmark & --         & DC & 1.6849 & 1.4421 & 2.279 & 1.1502 & 2.8580 & 3.1483 \\
		\checkmark & HypHC      & DC & 1.6852 & 1.4242 & 2.274 & 1.1394 & 2.8251 & 3.1389 \\
		\checkmark & Set-HypHC  & DC & \textbf{1.6801} & \textbf{1.4168} & \textbf{2.273} & \textbf{1.0853} & \textbf{2.8085} & \textbf{3.1380} \\
		\bottomrule
	\end{tabular}

	\vspace{1.5mm}
	\begin{minipage}{0.96\textwidth}
		\footnotesize
		\emph{Note:} $\checkmark$ indicates embeddings from SARL; otherwise, raw normalized features are used.
		``--'' indicates no HC optimization. Lower DC is better.
	\end{minipage}
\end{table*}

\subsection{Discussion}

\subsubsection{Analysis of Learned Set-Aware Representations}

We first examine whether SARL learns an embedding space suitable for set construction, using Isolet1 as a representative dataset under the same setting as in Section~\ref{sec:comparison}. Fig.~\ref{fig:isolet1_tsne_clean} shows that the learned embeddings form compact and well-separated local groups. In Fig.~\ref{fig:isolet1_anchor}, around a representative anchor, ML neighbors are pulled close, whereas CL neighbors are pushed away, indicating that the set-aware loss encodes the desired merge preference. The constructed sets in Fig.~\ref{fig:isolet1_tsne_sets} further align with local same-class groups. They serve as reliable local structural units for Set-HypHC.

Fig.~\ref{fig:distance_distribution} compares the CDFs of normalized intra-class and inter-class distances before and after SARL. After representation learning, the intra-class CDF rises faster at smaller distances, while the inter-class CDF shifts toward larger distances. This enlarged gap shows that SARL improves the class-consistent geometry of the embedding space, making local neighborhoods more reliable for subsequent set construction.

\begin{figure}[t]
	\centering
	\subfloat[Labels\label{fig:isolet1_tsne_clean}]{%
		\includegraphics[width=0.33\columnwidth]{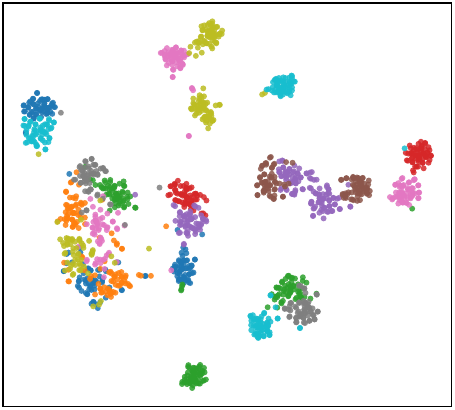}%
	}
	\hfil
	\subfloat[Anchor constraints\label{fig:isolet1_anchor}]{%
		\includegraphics[width=0.33\columnwidth]{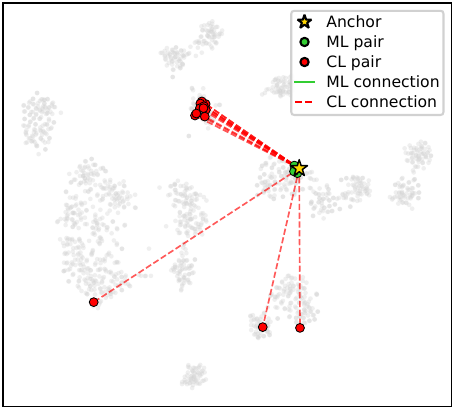}%
	}
    \hfil
	\subfloat[Constructed sets\label{fig:isolet1_tsne_sets}]{%
		\includegraphics[width=0.33\columnwidth]{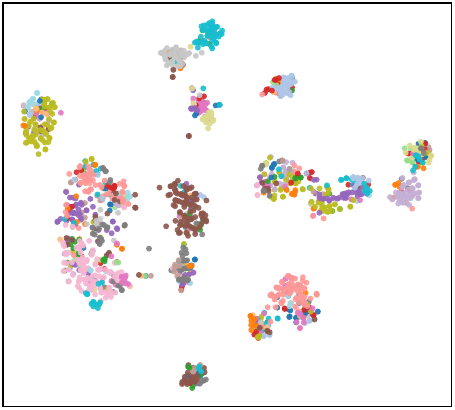}%
	}
	
	\caption{Visualization of learned set-aware representations on Isolet1. 
		(a) t-SNE colored by ground truth labels. 
		(b) Local constraint view around an anchor sample, with ML pairs in green and CL pairs in red. 
		(c) t-SNE colored by constructed sets.}
	\label{fig:representation_visualization}
\end{figure}

\begin{figure}[t]
	\centering
	\subfloat[Original feature space\label{fig:isolet1_orig_cdf}]{%
		\includegraphics[width=0.48\columnwidth]{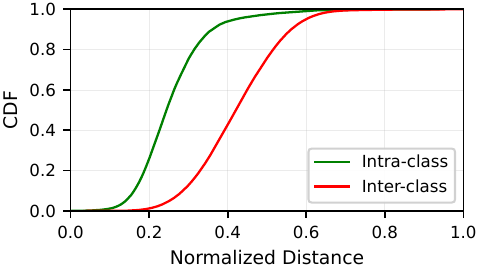}%
	}
	\hfil
	\subfloat[Learned embedding space\label{fig:isolet1_emb_cdf}]{%
		\includegraphics[width=0.48\columnwidth]{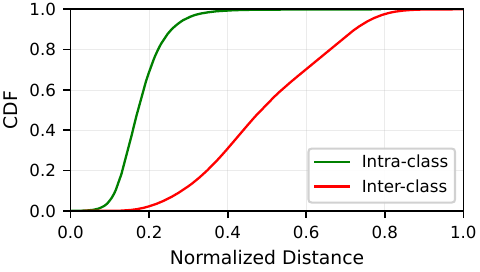}%
	}
	
	\caption{CDFs of normalized intra-class and inter-class distances on Isolet1.
(a) Original feature space.
(b) Learned hyperbolic embedding space after SARL.}
	\label{fig:distance_distribution}
\end{figure}

\begin{figure}[t]
	\subfloat[Isolet1\label{fig:isolet1_dp_curve}]{%
		\includegraphics[width=0.24\textwidth]{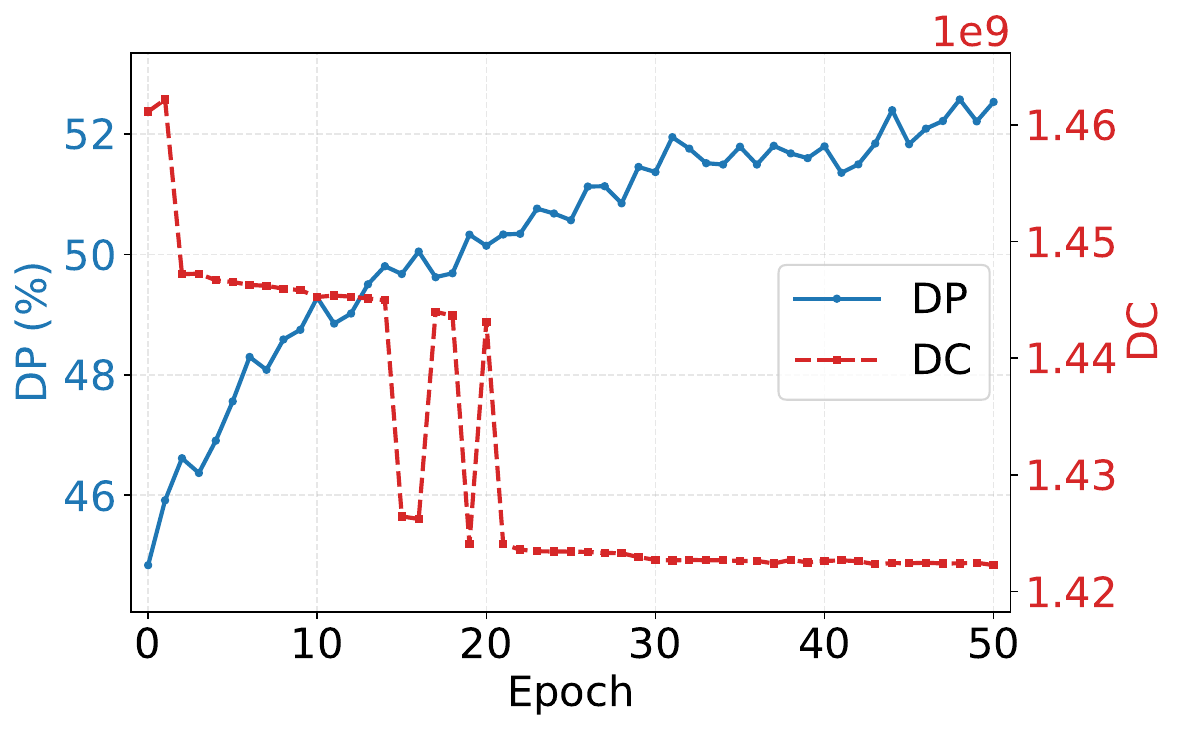}%
	}
	\hfill
	\subfloat[Australian\label{fig:Australian_dp_curve}]{%
		\includegraphics[width=0.24\textwidth]{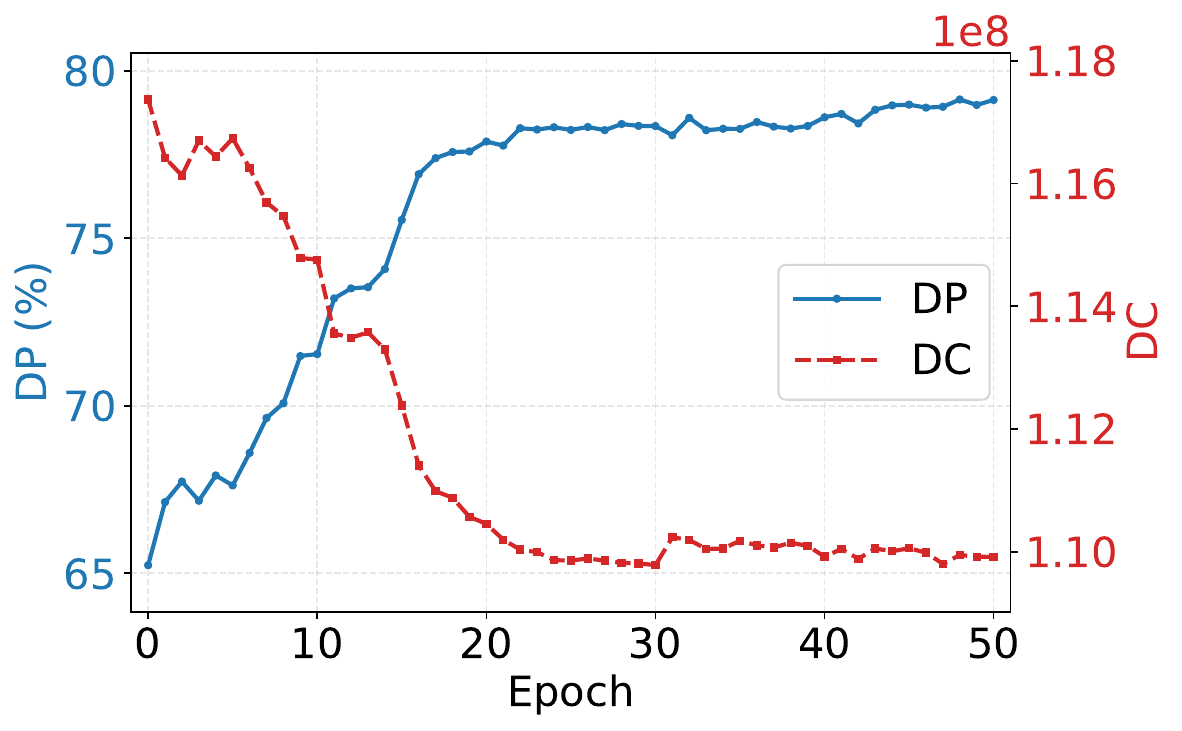}%
	}
	\caption{Evolution of DP and DC during set-based HypHC training. The initial values are computed with the embeddings from the hyperbolic autoencoder.}
	\label{DP_cost}
\end{figure}

\begin{figure}[t]
	\centering
	\subfloat[Australian]{%
		\includegraphics[
		width=0.4\columnwidth,
		trim=0 0 0 30pt,
		clip
		]{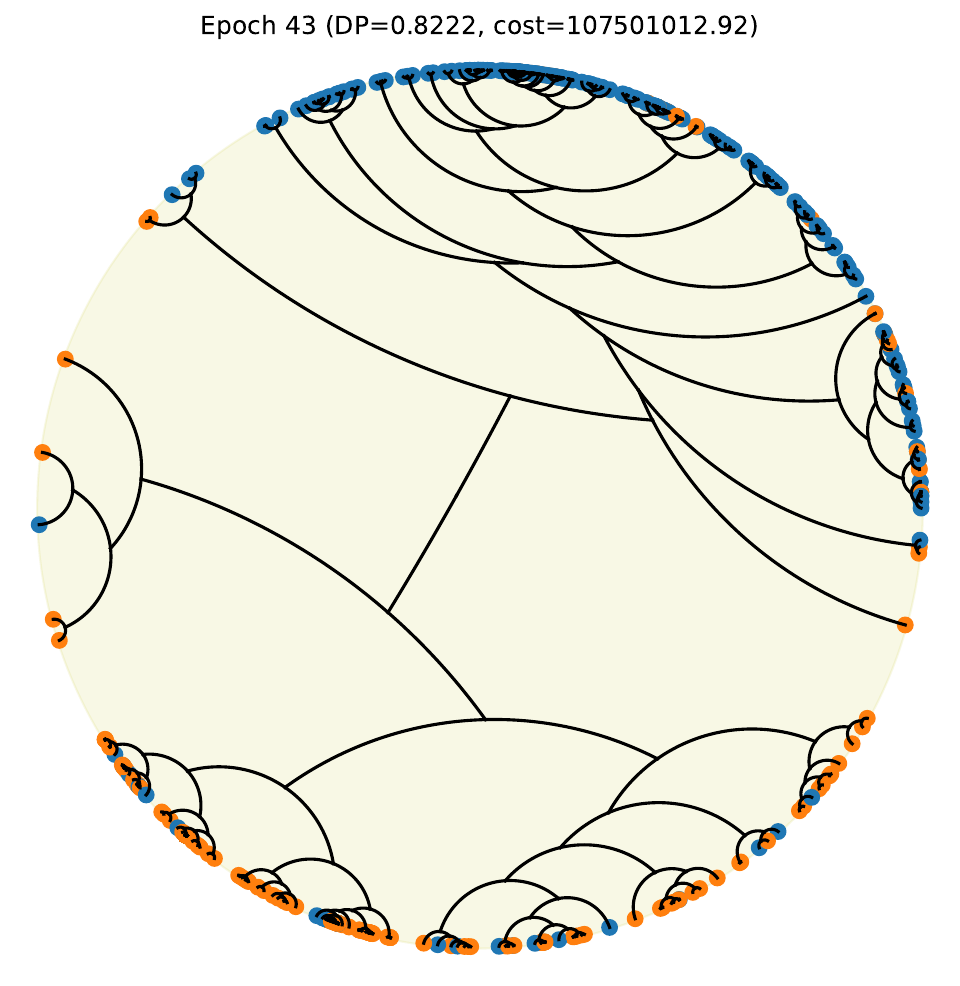}%
	}
	\hfill
	\subfloat[Wine]{%
		\includegraphics[
		width=0.4\columnwidth,
		trim=0 0 0 33pt,
		clip
		]{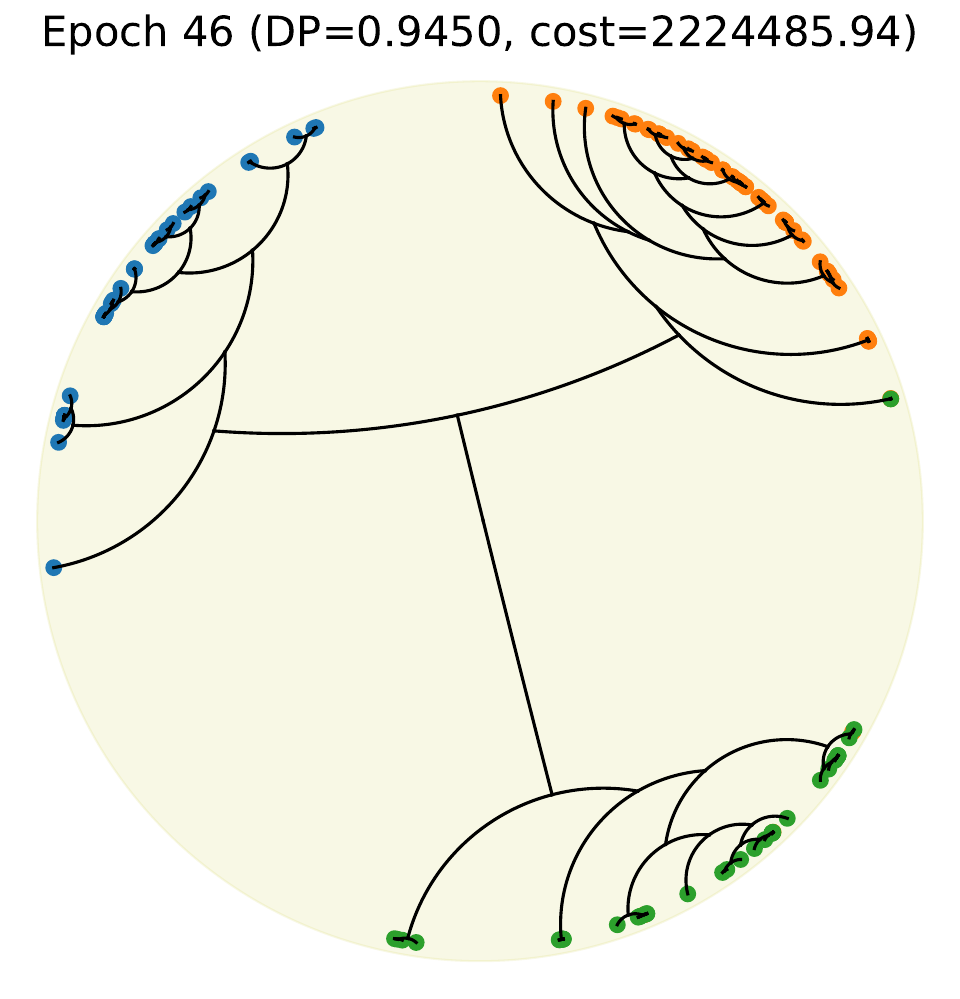}%
	}
	\caption{Visualization of decoded hierarchical trees after Set-HypHC optimization on Australian and Wine. The tree structures are visualized in the Poincar\'e model, with leaf nodes colored by ground-truth labels.}
	\label{fig:tree_evolution}
\end{figure}

\begin{figure*}[t]
	\centering
	\subfloat[Yale\label{fig:sensitivity_constraint_yale}]{%
		\includegraphics[width=0.24\textwidth]{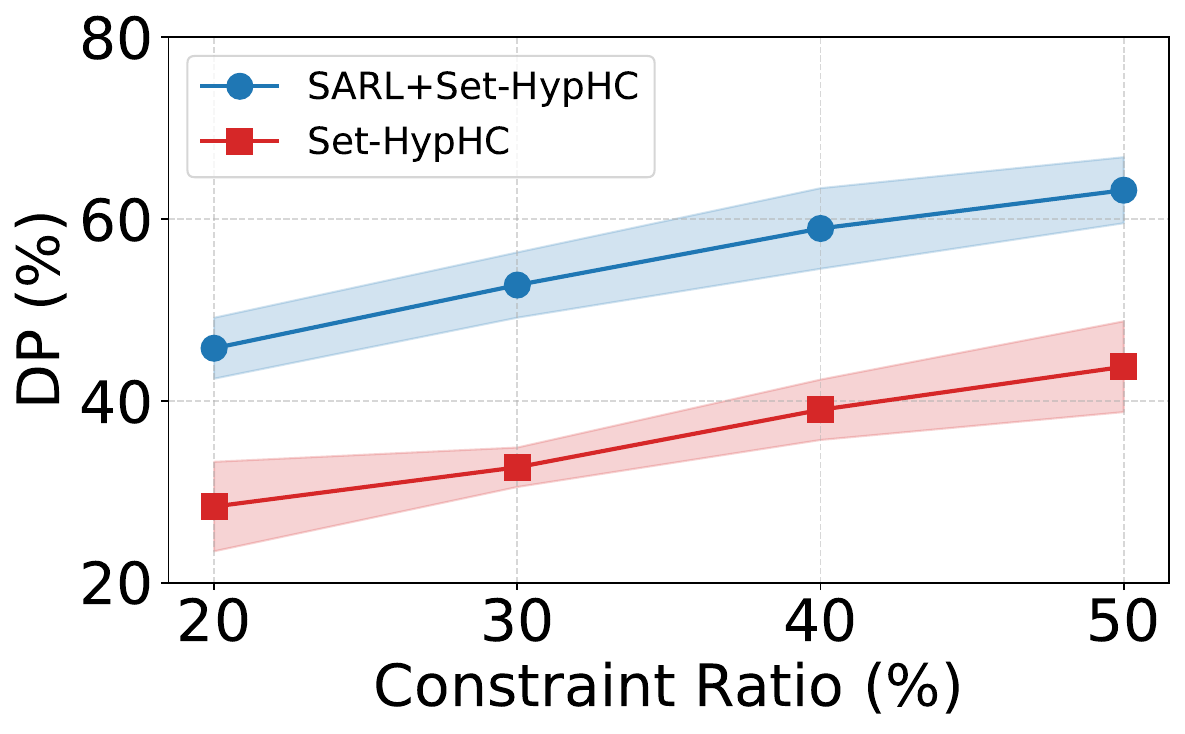}%
	}
	\hfill
	\subfloat[ORL\label{fig:sensitivity_constraint_orl}]{%
		\includegraphics[width=0.24\textwidth]{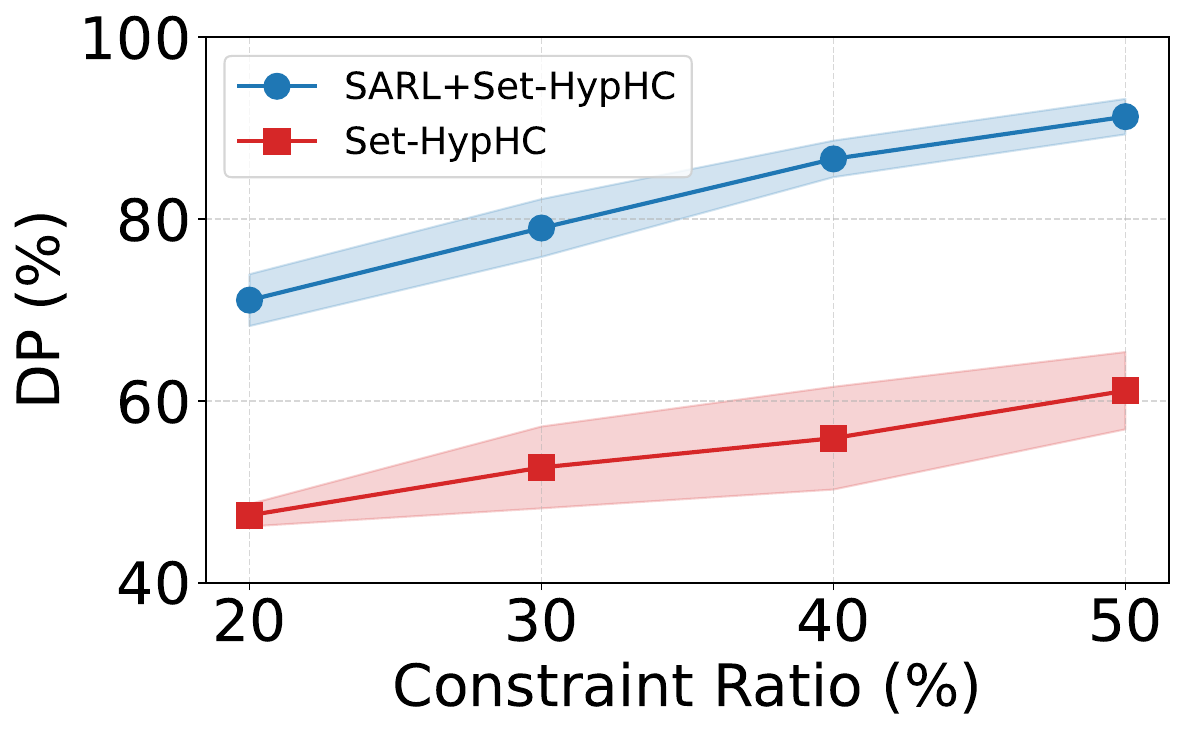}%
	}
	\hfill
	\subfloat[Isolet1\label{fig:sensitivity_constraint_isolet1}]{%
		\includegraphics[width=0.24\textwidth]{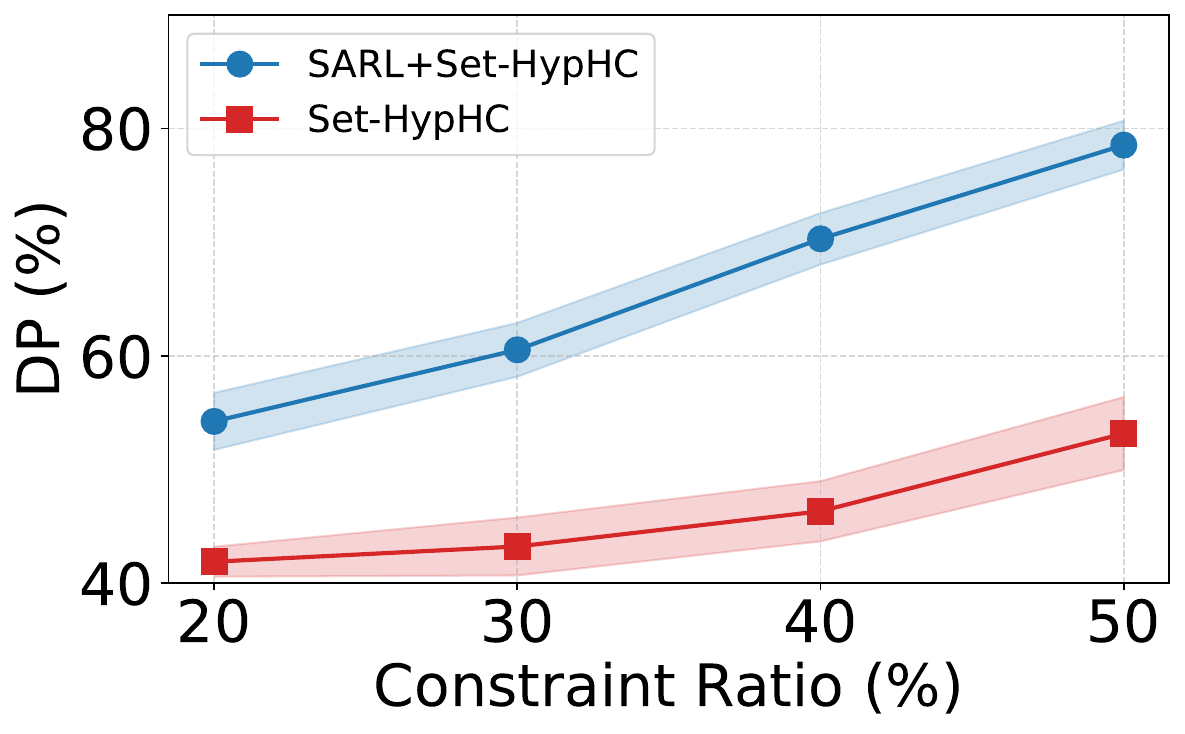}%
	}
	\hfill
	\subfloat[USPS\label{fig:sensitivity_constraint_usps}]{%
		\includegraphics[width=0.24\textwidth]{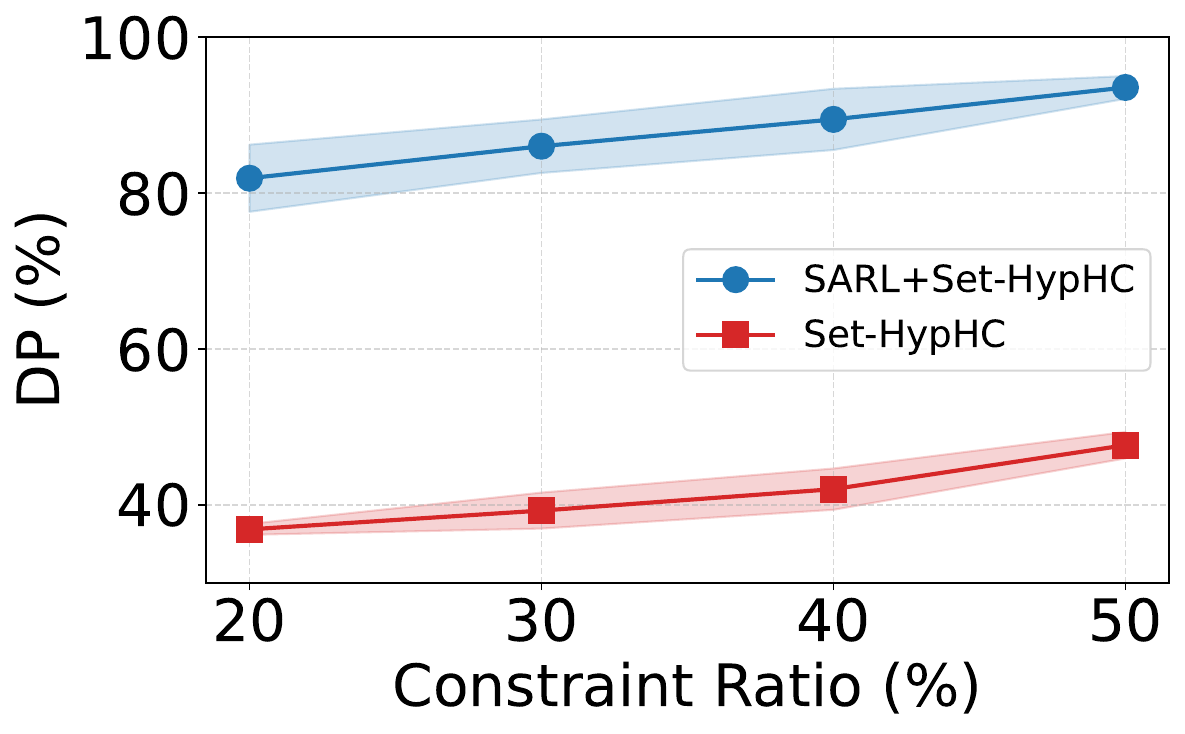}%
	}
	\caption{Performance of our method with different pairwise constraint amounts, which are controlled by constraint ratios from $\{20\%,30\%,40\%,50\%\}$.}
	\label{fig:sensitivity_constraint}
\end{figure*}

\begin{figure*}[!t]
	\centering
	
	\subfloat[Yale\label{fig:sensitivity_w_yale}]{%
		\includegraphics[width=0.22\textwidth]{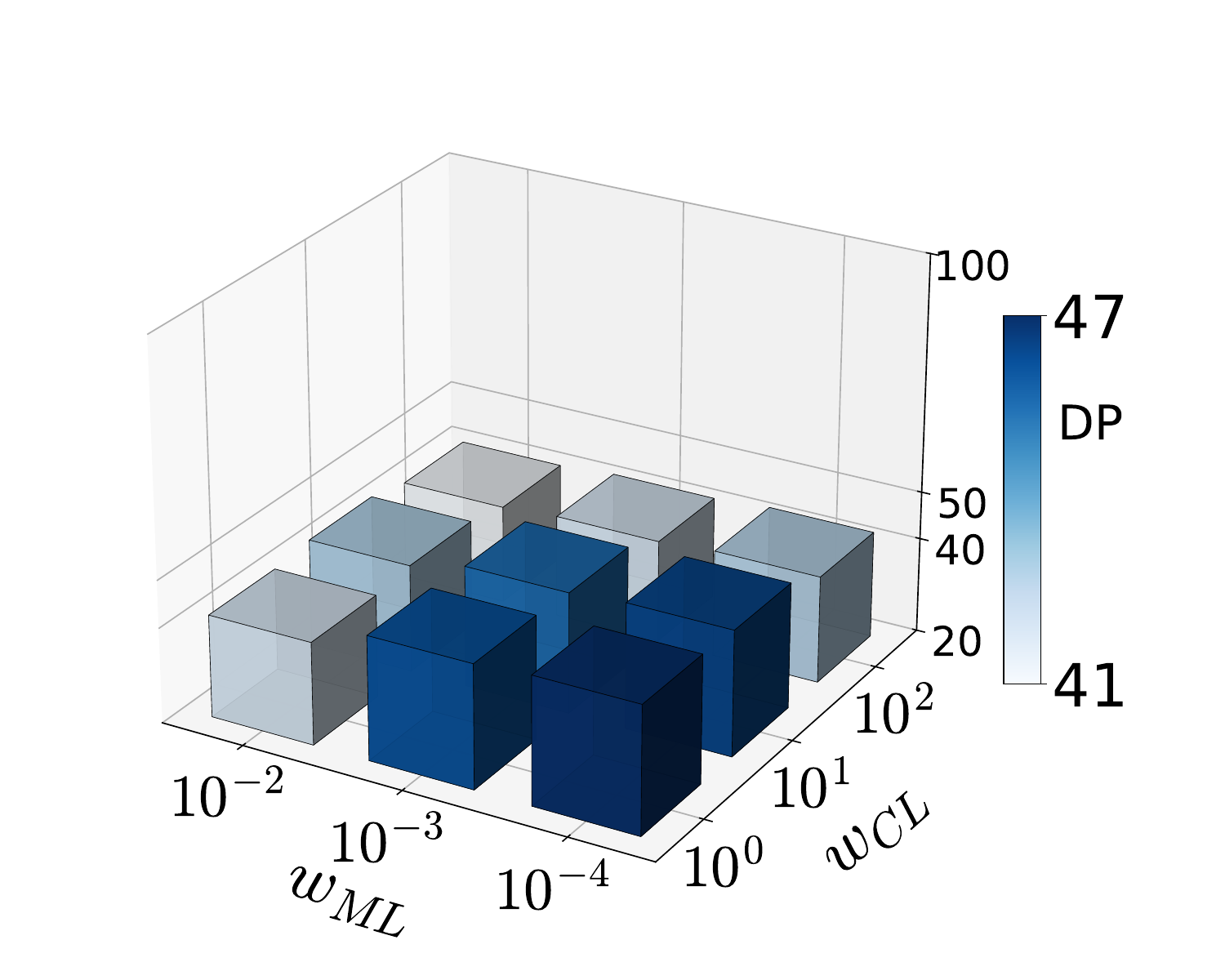}
	}
	\hspace{0.01\textwidth}
	\subfloat[Australian\label{fig:sensitivity_w_australian}]{%
		\includegraphics[width=0.22\textwidth]{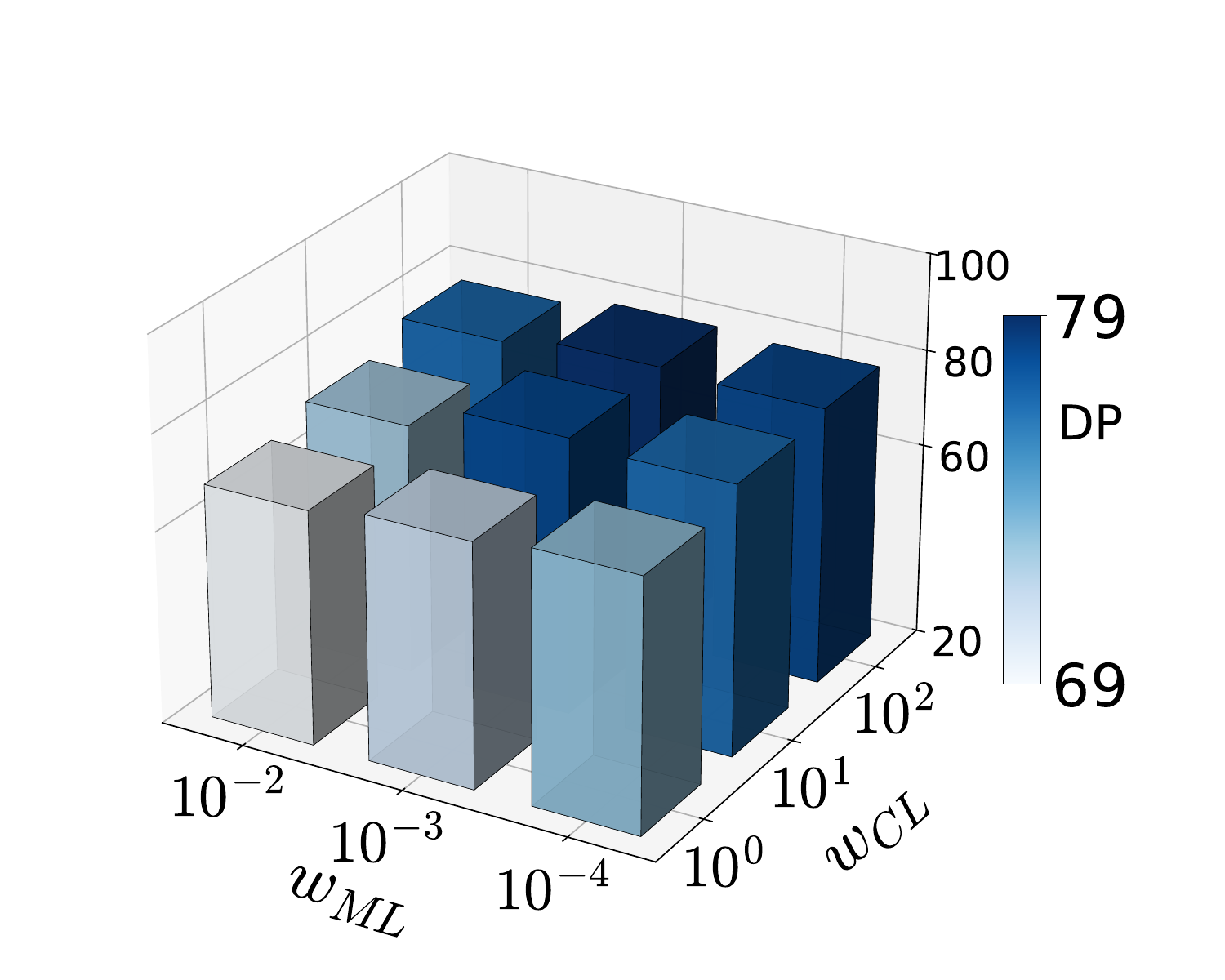}
	}
	\hspace{0.01\textwidth}
	\subfloat[Isolet1\label{fig:sensitivity_w_isolet1}]{%
		\includegraphics[width=0.22\textwidth]{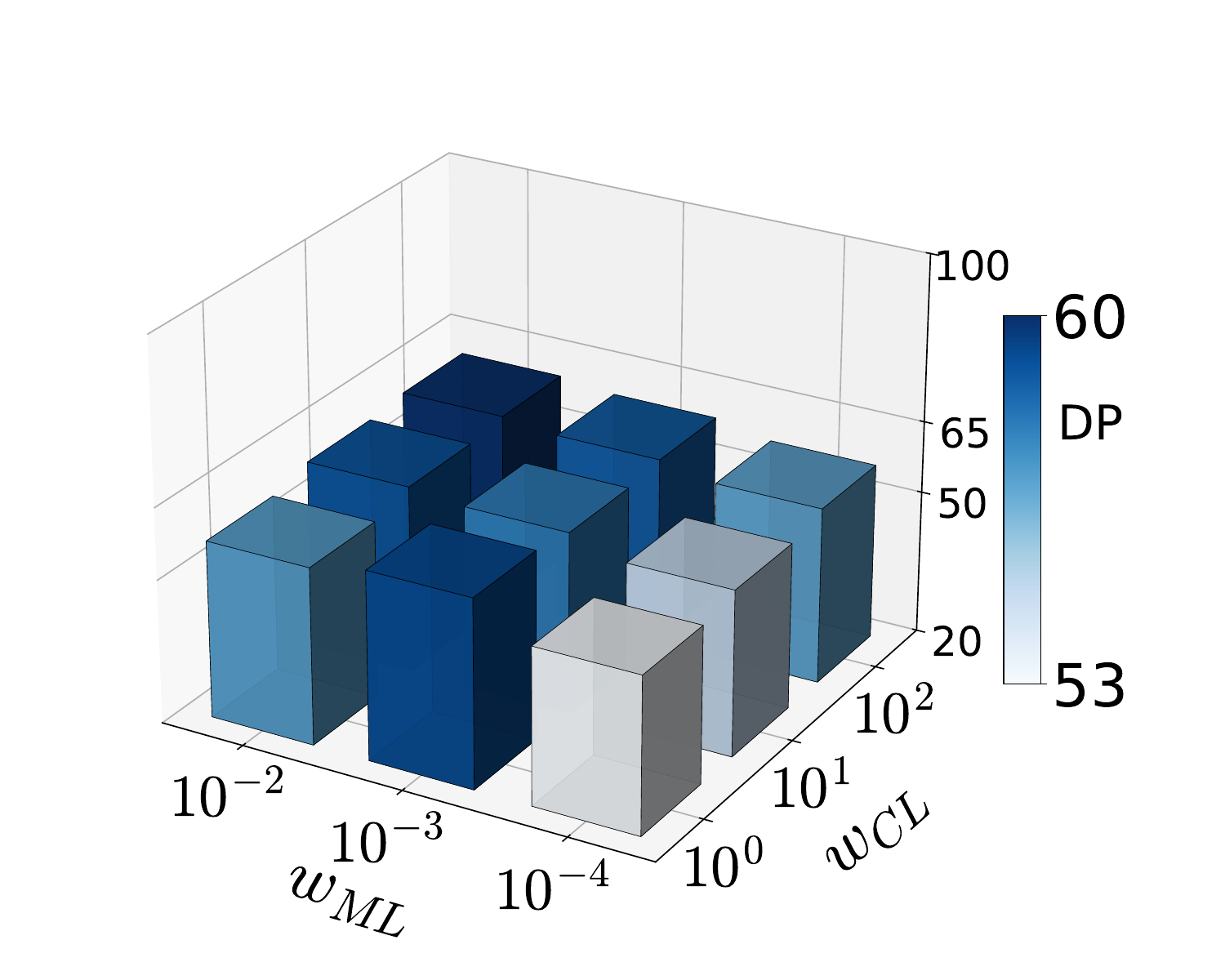}
	}
	\hspace{0.01\textwidth}
	\subfloat[Spambase\label{fig:sensitivity_w_Spambase}]{%
		\includegraphics[width=0.22\textwidth]{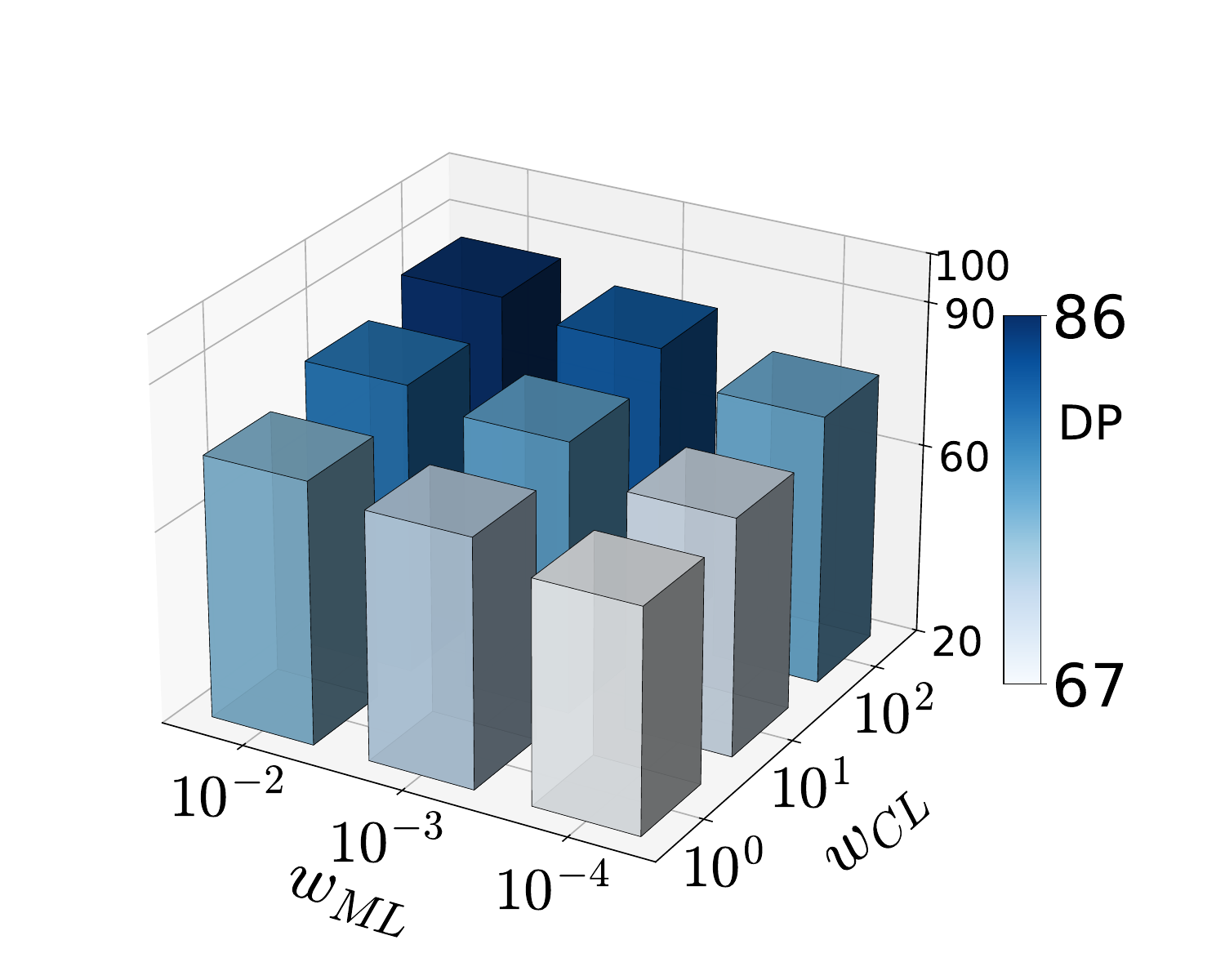}
	}
	
	\vspace{1mm}
	
	\subfloat[Yale\label{fig:sensitivity_k_yale}]{%
		\includegraphics[width=0.18\textwidth]{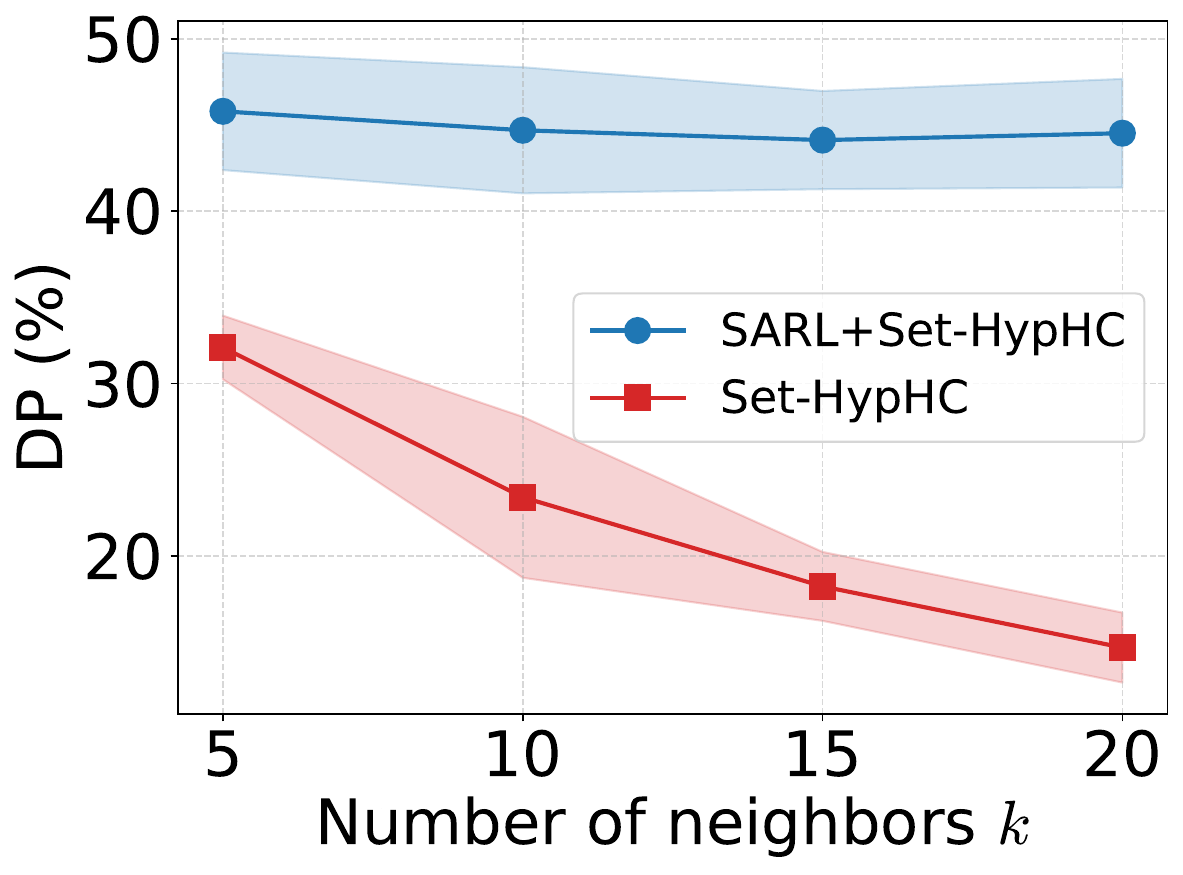}
	}
	\hspace{0.005\textwidth}
	\subfloat[Australian\label{fig:sensitivity_k_australian}]{%
		\includegraphics[width=0.18\textwidth]{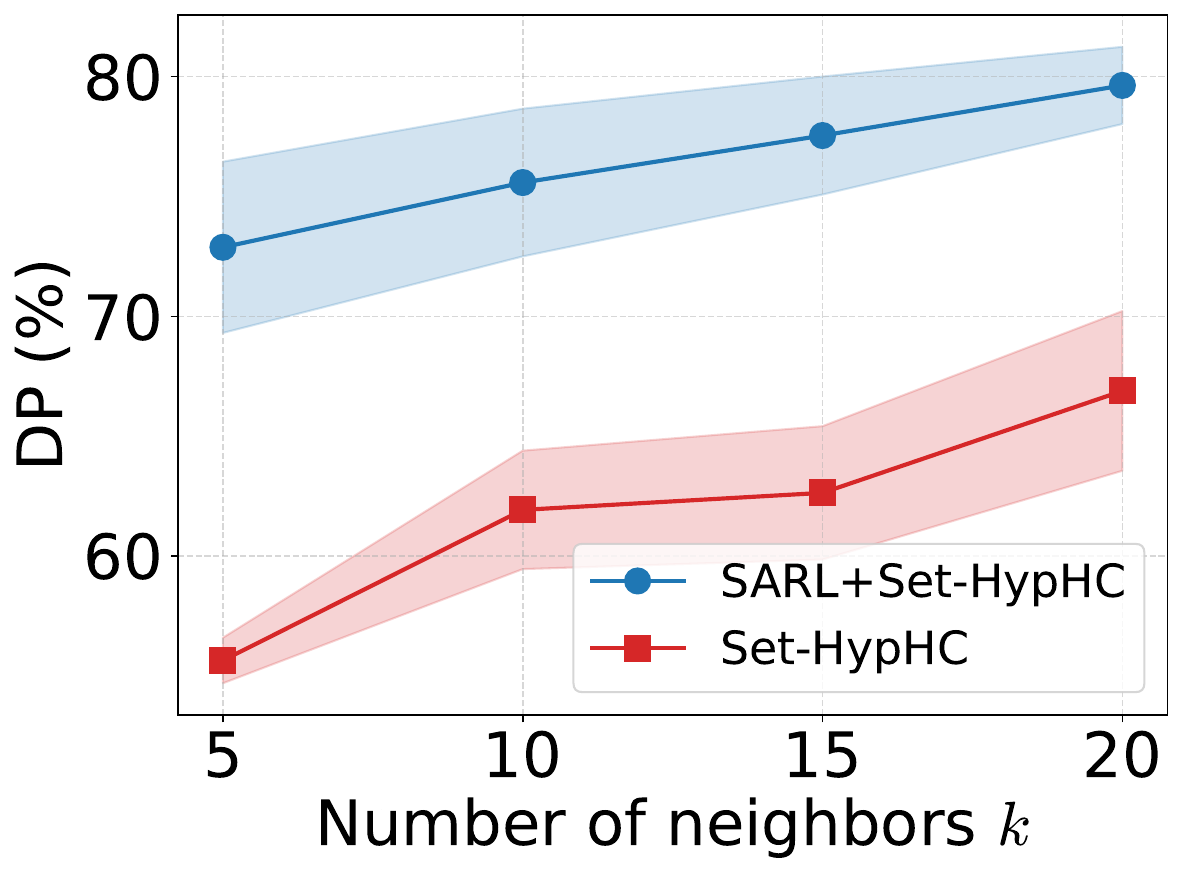}
	}
	\hspace{0.005\textwidth}
	\subfloat[Isolet1\label{fig:sensitivity_k_isolet1}]{%
		\includegraphics[width=0.18\textwidth]{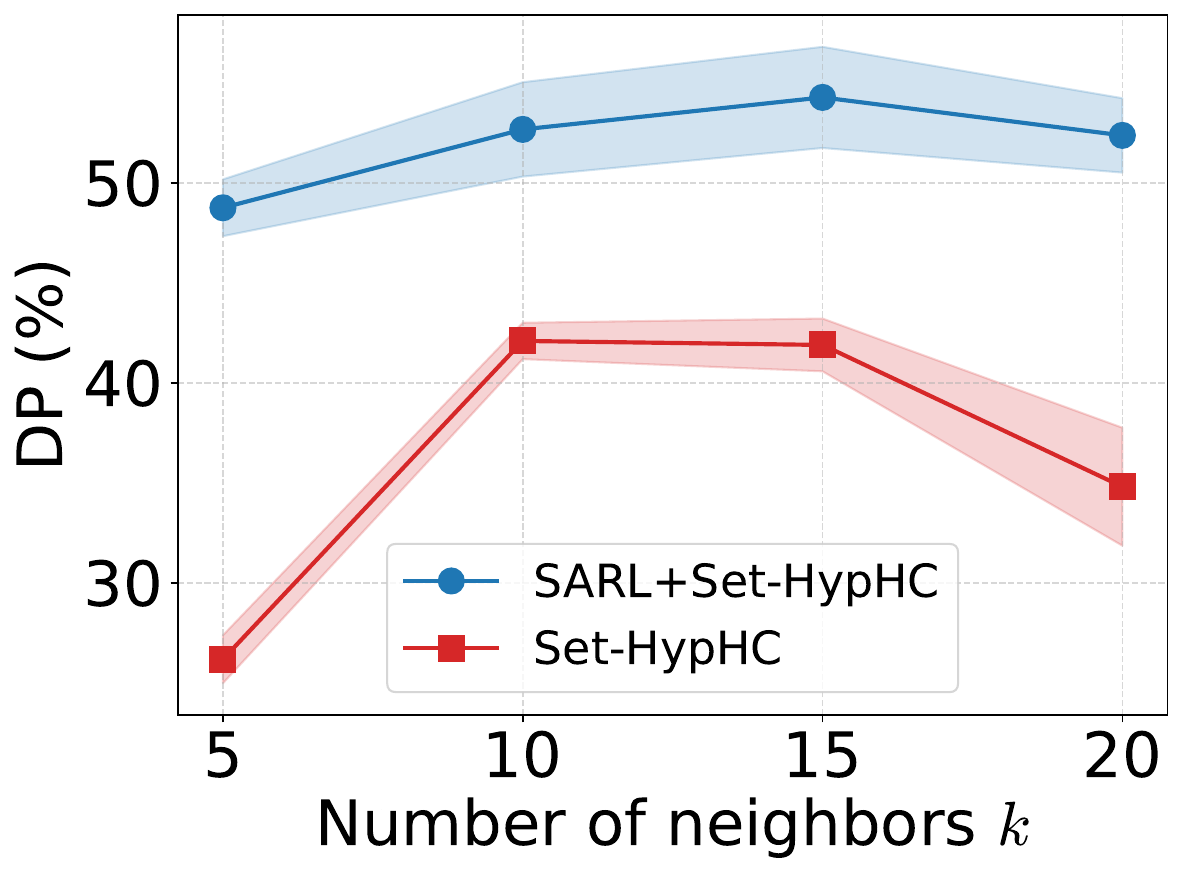}
	}
	\hspace{0.01\textwidth}
	\subfloat[Temperature $t$\label{fig:sensitivity_t}]{%
		\includegraphics[width=0.34\textwidth]{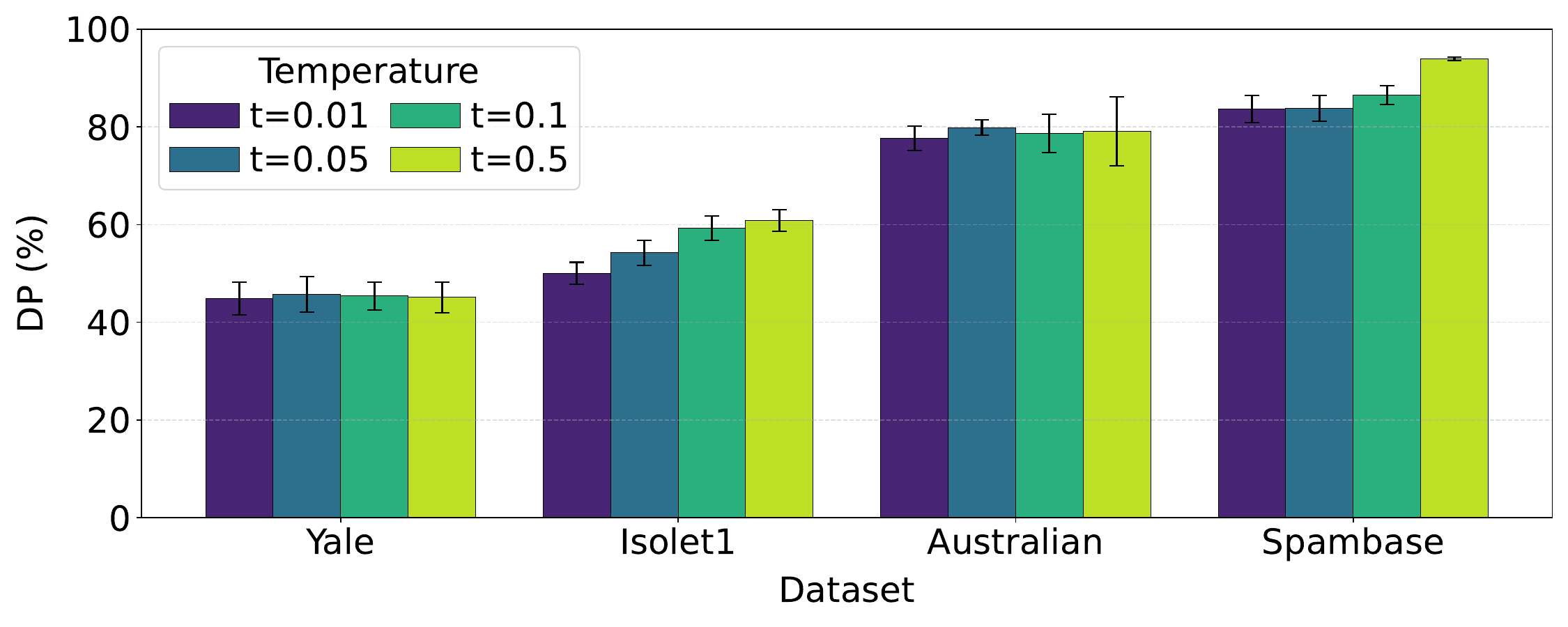}
	}
	
	\caption{Parameter sensitivity. (a)--(d) Balancing coefficients $w_{\mathrm{ML}}$ and $w_{\mathrm{CL}}$. (e)--(g) Number of neighbors $k$. (h) Temperature $t$.}
	\label{fig:parameter_sensitivity}
\end{figure*}

\subsubsection{Analysis of Set-Based HypHC}

We next analyze how Set-HypHC refines the hierarchy. Fig.~\ref{DP_cost} shows the evolution of DP and DC on Isolet1 and Australian, where the initial values are computed from the hyperbolic autoencoder embeddings before Set-HypHC optimization. During training, DP generally increases and DC decreases, indicating that Set-HypHC improves both label consistency and similarity-based tree quality. This supports the effectiveness of optimizing subtree merge preferences through intra-set LCAs.

Fig.~\ref{fig:tree_evolution} visualizes the decoded trees after optimization on Australian and Wine. Samples with the same ground-truth label are largely organized into coherent local branches, showing that the optimized hierarchy captures label-consistent subtree structures. The decoded trees also exhibit a clear multi-level organization in the Poincar\'e model, where different class-related subtrees are separated with balanced spatial layouts rather than being collapsed into a few crowded regions.

\subsubsection{Effect of Constraint Ratio}

We further study how the amount of pairwise supervision affects the proposed framework. As shown in Fig.~\ref{fig:sensitivity_constraint}, increasing the constraint ratio improves Set-HypHC in both settings, with and without SARL. Since the setting without SARL does not modify the original similarity structure, this improvement indicates that richer constraints provide more reliable constraint-induced sets and inter-set similarities.

Table~\ref{tab:set_purity} further reports the weighted purity and number of non-singleton sets in the case with SARL. For non-singleton sets, the weighted purity is computed as the size-weighted average of set purity, where the purity of each set is defined by the proportion of samples from its dominant class. The weighted purity generally increases as the constraint ratio grows, indicating that additional supervision improves the class consistency of the constructed sets. In contrast, the number of non-singleton sets tends to decrease or fluctuate across datasets, suggesting that the performance gain in Fig.~\ref{fig:sensitivity_constraint} mainly comes from more reliable local structural units rather than from simply absorbing more samples into non-singleton sets.

\begin{table}[t]
	\centering
	\caption{Weighted purity (\%) and number of non-singleton sets.}
	\label{tab:set_purity}
	\begin{tabular}{lcccc}
		\toprule
		\textbf{CR} & \textbf{Yale} & \textbf{ORL} & \textbf{Isolet1} & \textbf{USPS}\\
		\midrule
		20\% & 69.75\% (35) & 89.18\% (86) & 84.80\% (269) & 96.62\% (1754) \\
		30\% & 77.30\% (34) & 93.75\% (83) & 91.76\% (306) & 96.31\% (1768) \\
		40\% & 82.93\% (32) & 96.22\% (86) & 93.91\% (282) & 98.12\% (1643) \\
		50\% & 86.06\% (33) & 95.97\% (75) & 94.74\% (235) & 99.09\% (1404) \\
		\bottomrule
	\end{tabular}
	
	\vspace{1.0mm}
	\begin{minipage}{0.96\columnwidth}
		\footnotesize
\emph{Note:} Parentheses indicate the number of non-singleton sets. Weighted purity is computed over non-singleton sets. CR denotes the constraint ratio.
	\end{minipage}
\end{table}

\subsection{Parameter Sensitivity}

We further evaluate the sensitivity of the proposed method to the balancing coefficients, neighborhood size, and temperature coefficient. As shown in Figs.~\ref{fig:sensitivity_w_yale}--\ref{fig:sensitivity_w_Spambase}, the performance remains stable under most combinations of $w_{\mathrm{ML}}$ and $w_{\mathrm{CL}}$, indicating that the method does not rely on a narrowly tuned pair of loss weights. Figs.~\ref{fig:sensitivity_k_yale}--\ref{fig:sensitivity_k_isolet1} show that the variant without SARL is more sensitive to $k$, whereas the full pipeline remains more stable, suggesting that SARL improves the reliability of local neighborhood construction. Finally, Fig.~\ref{fig:sensitivity_t} shows that DP remains competitive across different temperatures, indicating that the set-level objective is not overly sensitive to $t$.

\section{Conclusion}

This paper proposed a semi-supervised hyperbolic HC framework with set-level structural priors. The proposed method bridges leaf-level pairwise supervision and subtree-level hierarchy optimization by constructing constraint-induced sets as soft carriers of structural guidance, rather than fixed clusters or predetermined subtrees. It first learns constraint-consistent hyperbolic embeddings through SARL, then constructs sets and computes inter-set similarities from pairwise constraints and the learned similarity structure, and finally optimizes set-level triplet relations through Set-HypHC on the Poincar\'e model. Experiments on benchmark datasets show that the proposed framework consistently outperforms representative semi-supervised and unsupervised baselines. Further analyses confirm that the proposed stages jointly improve label consistency and similarity-based tree quality by providing reliable set-level structural priors.

\bibliographystyle{IEEEtran}
\bibliography{refs}

\end{document}